\documentclass[runningheads]{llncs}

\PassOptionsToPackage{table}{xcolor}

\usepackage{eccv}

\usepackage{eccvabbrv}

\usepackage{graphicx}
\usepackage{booktabs}

\usepackage[accsupp]{axessibility}  

\usepackage{algorithm}
\usepackage{algorithmic}
\usepackage{amsmath}
\usepackage{amssymb}
\usepackage{tcolorbox}
\usepackage{wrapfig}
\usepackage{float}
\usepackage{siunitx}
\usepackage{multirow}

\usepackage{xcolor}
\usepackage{colortbl}
\usepackage{soul}
\usepackage{needspace}
\usepackage{subcaption}
\usepackage{arydshln}
\usepackage{pifont}
\usepackage{enumitem}
\usepackage{xparse}

\definecolor{eccvblue}{rgb}{0.12,0.49,0.85}
\definecolor{blue}{HTML}{ecf9fd}
\definecolor{gray}{HTML}{f4f4f4}
\definecolor{yellow}{HTML}{fefbe1}
\definecolor{green}{HTML}{e1fee4}
\definecolor{boxblue}{RGB}{26,52,190}
\definecolor{boxbg}{RGB}{245,247,255}

\newcommand{\bluebox}[2]{%
    \par\medskip
    \Needspace{\baselineskip}%
    \noindent
    \fcolorbox{boxblue}{boxbg}{%
        \parbox{\dimexpr\linewidth-2\fboxsep-2\fboxrule\relax}{%
            \colorbox{boxblue}{%
                \parbox{\dimexpr\linewidth-2\fboxsep\relax}{%
                    \color{white}\textbf{#1}%
                }%
            }%

            #2
        }%
    }%
    \par\medskip
}

\usepackage{hyperref}

\usepackage{orcidlink}

\makeatletter

\renewcommand*\l@section[2]{%
  \@dottedtocline{1}{0em}{2.3em}{\bfseries #1}{\bfseries #2}
}

\renewcommand*\l@subsection[2]{%
  \@dottedtocline{2}{2.0em}{3.2em}{#1}{#2}%
}

\newcommand{\supplementarytoc}{%
  \begingroup
  \hypersetup{linkcolor=black}%
  \setcounter{tocdepth}{2}%
  \begin{center}
    \begin{minipage}{0.95\textwidth}
      \small
      \@starttoc{atoc}%
    \end{minipage}
  \end{center}
  \endgroup
}

\let\origsection\section
\let\origsubsection\subsection

\newif\ifinappendixtoc
\inappendixtocfalse

\RenewDocumentCommand{\section}{s o m}{%
  \IfBooleanTF{#1}{%
    \IfNoValueTF{#2}
      {\origsection*{#3}}
      {\origsection*[#2]{#3}}%
  }{%
    \IfNoValueTF{#2}
      {\origsection{#3}}
      {\origsection[#2]{#3}}%
    \ifinappendixtoc
      \addcontentsline{atoc}{section}{\protect\numberline{\thesection}\IfNoValueTF{#2}{#3}{#2}}%
    \fi
  }%
}

\RenewDocumentCommand{\subsection}{s o m}{%
  \IfBooleanTF{#1}{%
    \IfNoValueTF{#2}
      {\origsubsection*{#3}}
      {\origsubsection*[#2]{#3}}%
  }{%
    \IfNoValueTF{#2}
      {\origsubsection{#3}}
      {\origsubsection[#2]{#3}}%
    \ifinappendixtoc
      \addcontentsline{atoc}{subsection}{\protect\numberline{\thesubsection}\IfNoValueTF{#2}{#3}{#2}}%
    \fi
  }%
}

\makeatother

\begin{document}

\title{TanGO: Training-Free 3D Editing via Tangent-space Guidance and Optimization}

\titlerunning{TanGO}

\author{Siwoo Lim\inst{1}\orcidlink{0009-0001-1187-6548} \and
Sunjae Yoon\inst{2}\orcidlink{0000-0001-7458-5273} \and
Gwanhyeong Koo\inst{1}\orcidlink{0009-0005-6455-3223} \and
Hyeonseo Yun\inst{1}\orcidlink{0009-0001-2486-1142} \and
Chang D. Yoo\inst{1}\thanks{Corresponding Author}\orcidlink{0000-0002-0756-7179}}

\authorrunning{S. Lim et al.}

\institute{Korea Advanced Institute of Science and Technology, Daejeon, Republic of Korea \\
\email{\{siu4450, kookie, hyunseo, cd\_yoo\}@kaist.ac.kr}
\and
Chung-Ang University, Seoul, Republic of Korea\\
\email{\{sunjaeyoon\}@cau.ac.kr}}

\maketitle

\begin{figure}[h]
  \vspace{-1cm}
  \centering
  \includegraphics[width=0.9\textwidth]{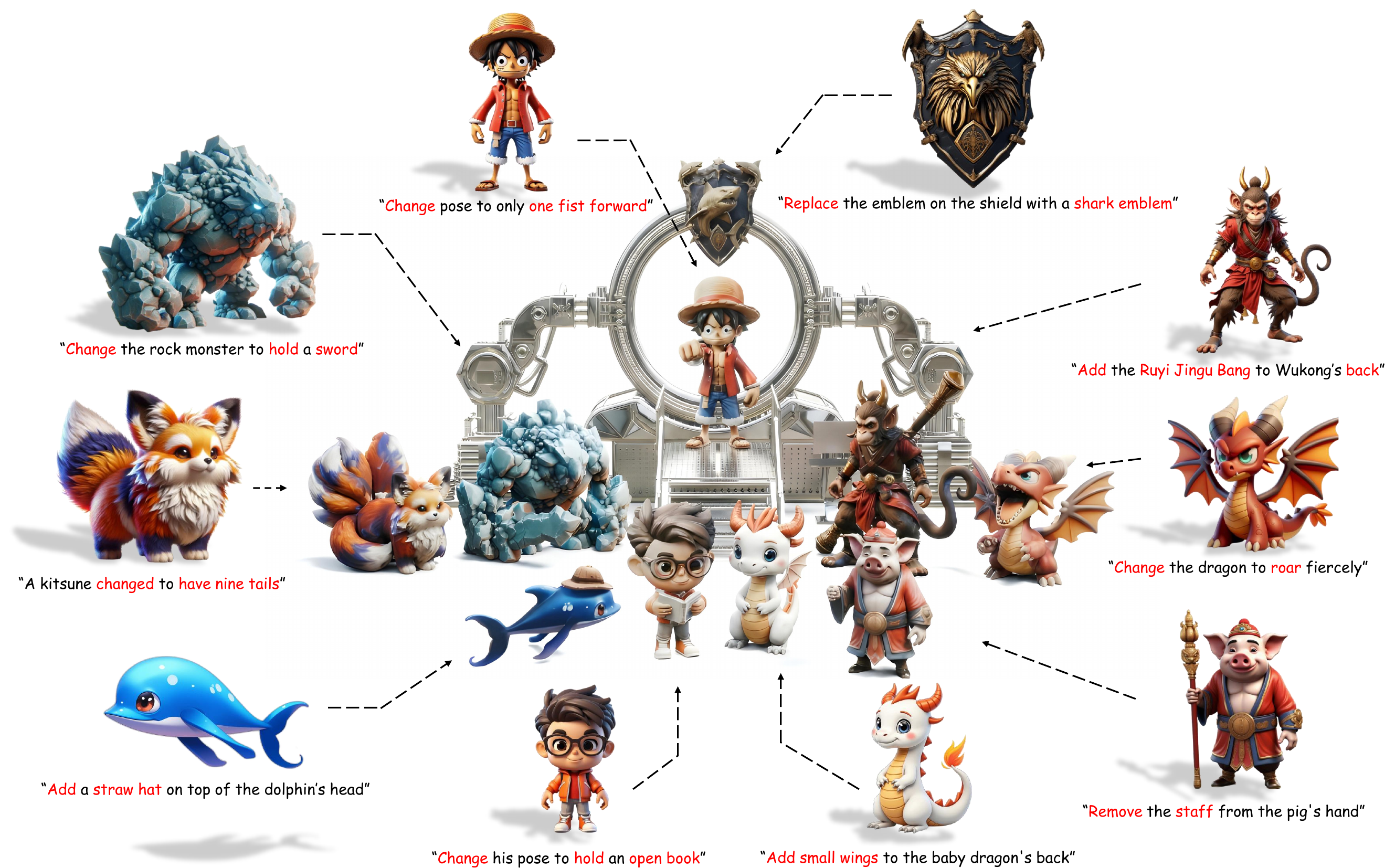}
  \caption{\textbf{Overview of 3D Editing Results.} TanGO achieves precise localized edits across diverse categories, preserving unedited geometry and source identity.}
  \label{fig:fig1}
  \vspace{-1cm}
\end{figure}

\begin{abstract}

While recent flow-matching 3D generative models (e.g., VecSet) adopt structured representations, their tokens share global context, causing conventional training-free editing to suffer from semantic artifacts such as collapsed preserved regions or incomplete transformations. To address this, we propose \textbf{TanGO}, a training-free framework that enables adaptive per-token steering in the tangent space of generative dynamics. To realize this selective control, we formulate a one-step optimal control rule and determine the strength of each token’s control signal using a von Mises-Fisher inspired directional discrepancy derived from the source and target velocity fields. Experiments show that TanGO substantially reduces structural artifacts and achieves state-of-the-art performance, outperforming existing 3D editing baselines. The code is publicly available at \url{https://github.com/siw00-lim/TanGO}.


\end{abstract}

\begin{figure*}[t!]
  \begin{center}
    \centerline{\includegraphics[width=0.77\textwidth]{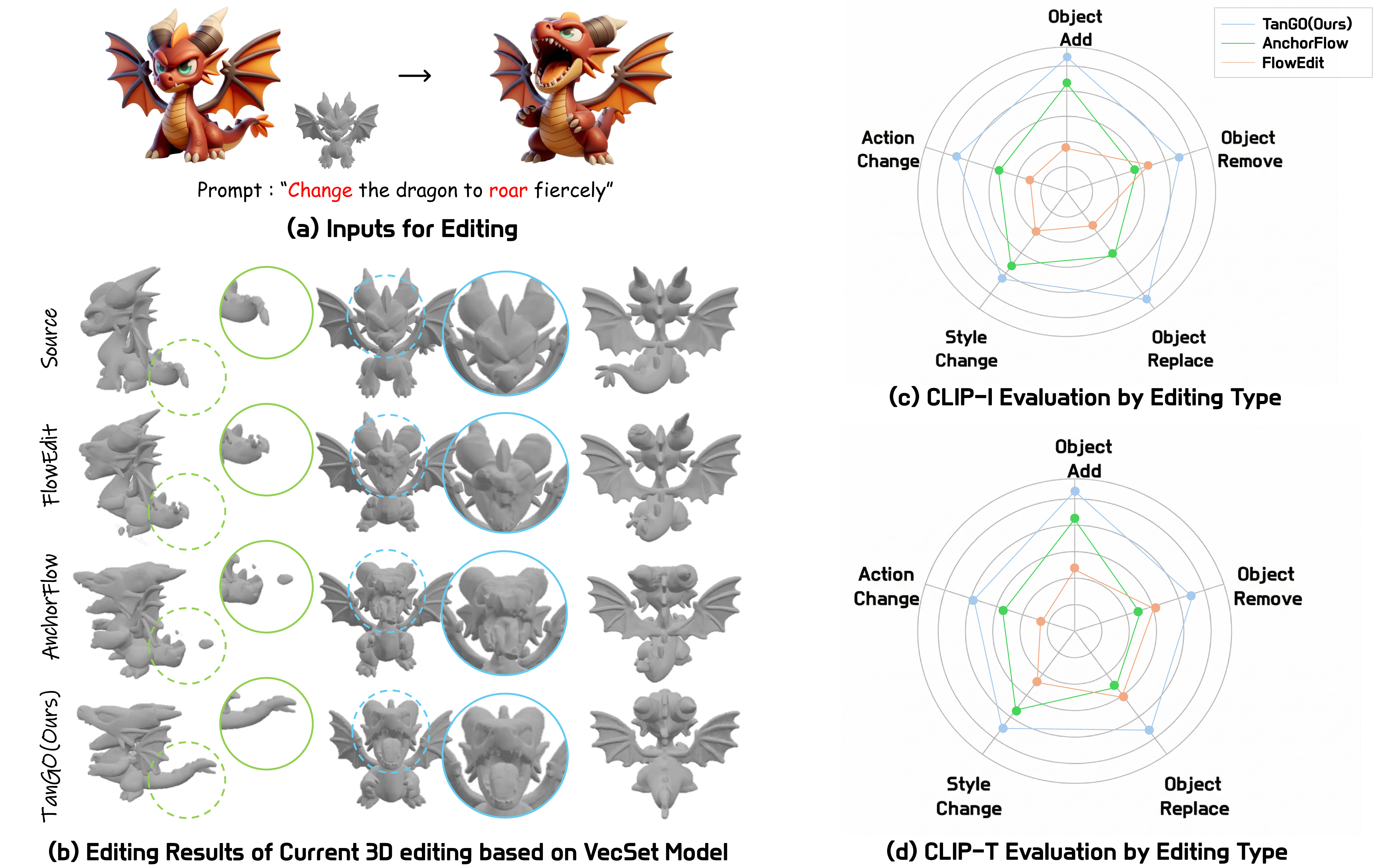}}\
    \caption{(a) Input for text-driven 3D editing. (b) Prior methods often introduce semantic artifacts, such as incomplete edits in target regions or structural degradation in preserved areas. In contrast, TanGO preserves unedited regions while achieving precise, localized edits. (c-d) CLIP-I and CLIP-T scores across five editing categories show that TanGO consistently improves the balance between localized preservation and semantic modification over existing baselines.}
    \label{fig:fig2}
  \end{center}
\end{figure*}

\section{Introduction}
Recently, 3D editing technology has become an essential tool in digital asset creation, with its impact stretching from gaming and entertainment to architecture and industrial design \cite{chen2024generic, bar2025editp23, haque2023instruct}. As digital twins and virtual environments become more sophisticated, there is a growing demand for tools that can precisely modify specific parts of a 3D model without having to rebuild the entire asset from scratch. Alongside this trend, the success of 3D generative models \cite{poole2022dreamfusion, lin2023magic3d, chen2023fantasia3d, hong2023lrm, wu2024direct3d, liang2024luciddreamer} has greatly accelerated the field by providing a strong foundation for high-quality editing. This has allowed the technology to move beyond simple shape changes toward smart, semantic-aware modifications that understand the underlying structure of 3D content \cite{michel2022text2mesh,chen2024gaussianeditor, gao2023textdeformer}.

Even with these technical advancements, 3D editing still faces practical challenges. Most existing methods rely on low-level representations like voxels and fail to utilize the structured space of advanced high-level representations such as VecSet\cite{zhang20233dshape2vecset, hunyuan3d2025hunyuan3d, li2025triposg, zhang2024clay, chen2025dora, lai2025lattice}. The main issue is that applying traditional editing methods to these high-level formats often ignores their natural semantic boundaries. As shown in Fig.~\ref{fig:fig2}(b), results edited by existing methods such as the 2D editing framework FlowEdit\cite{kulikov2025flowedit} or the 3D-based editing method AnchorFlow\cite{zhou2026anchorflow} suffer from significant semantic artifacts, including mesh tearing or surface collapsing in edited structures. Furthermore, these methods often fail to preserve the original source identity, even in structures that should be preserved.

\begin{figure*}[t!]
  \begin{center}
    \centerline{\includegraphics[width=\textwidth]{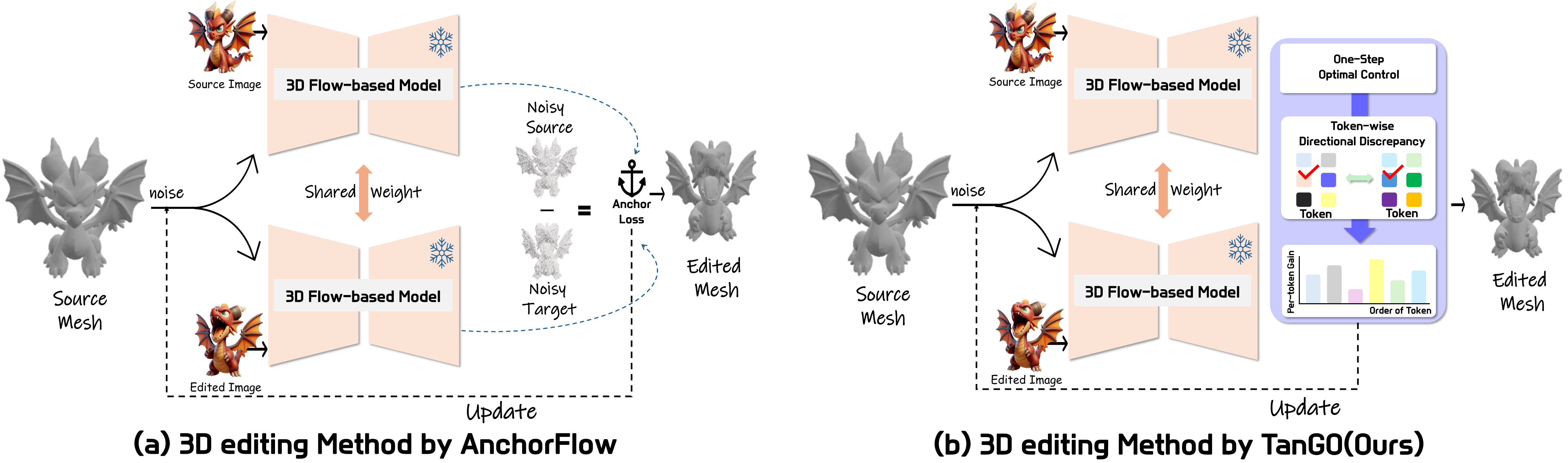}}\
    \caption{\textbf{Comparison of 3D editing frameworks.} (a) AnchorFlow optimizes an anchor-alignment loss to update the edited mesh. (b) TanGO (Ours) formulates the editing process as a one-step optimal control problem to modulate update signals at the token level. Our formulation selectively amplifies signals for tokens in regions designated for editing while suppressing those in regions to be preserved. To achieve this, we compute the token-wise directional discrepancy between the source and target velocities, and use it to derive an adaptive gain that scales the update intensity according to local editing requirements.}
    \label{fig:fig3}
  \end{center}
\end{figure*}

To demonstrate that these issues are not limited to specific samples, we conducted a quantitative evaluation using various types of editing datasets as illustrated in Fig.~\ref{fig:fig2}(c,d). We introduced CLIP-I\cite{radford2021learning} and CLIP-T\cite{radford2021learning} scores to measure semantic alignment. The results show that existing methods fail to achieve consistently high performance across all data types, and the scores themselves remain significantly low. This quantitative result confirms the persistent mismatch between the text prompts and the final 3D outputs in prior methods.

For a deeper understanding of why these artifacts occur when applying prior editing methods to high-level representations, we analyze the Hunyuan3D 2.1 model \cite{hunyuan3d2025hunyuan3d}. Our findings reveal that while tokens interact globally, each token's geometric influence is highly localized. This indicates that the imprecision of prior methods stems from the absence of explicit per-token actuation. Therefore, successful editing requires selective control. Instead of relying on a uniform global anchor loss, we propose \textbf{TanGO}, a training-free framework for localized 3D editing that allocates adaptive per-token steering in the tangent space (Fig.~\ref{fig:fig3}).

Specifically, TanGO formulates editing as an instantaneous one-step optimal control problem, yielding a simple closed-form update that balances editing progress and control energy. To achieve this, we define a vMF-grounded token-wise directional discrepancy between the normalized source and target velocities, which determines each token’s steering strength. This enables adaptive per-token steering via modulated gains and mean-gain normalization.

In summary, TanGO enables mask-free localized 3D editing by explicitly allocating control at the token level in tangent space. Additionally, to address the lack of structural diversity in existing benchmarks, we introduce the TanGOEdit dataset, which consists of both rigid and non-rigid editing tasks.

\section{Related Works}
\subsection{3D mesh generation}
Modern 3D generation has rapidly progressed from SDS-based lifting of 2D priors~\cite{poole2022dreamfusion, lin2023magic3d, chen2023fantasia3d, wang2023prolificdreamer} to native 3D foundation models based on diffusion and flow matching. Representative examples include TRELLIS~\cite{xiang2025structured}, which introduces structured 3D latents (SLAT) for scalable and versatile generation, as well as recent large-scale systems such as Hunyuan3D~\cite{hunyuan3d2025hunyuan3d}, TripoSG~\cite{li2025triposg}, and VecSet-based generators~\cite{zhang20233dshape2vecset} that improve quality, efficiency, and condition alignment. These models provide the representational backbone for downstream 3D editing.

\subsection{3D mesh editing}
3D editing is more challenging than 2D editing~\cite{brooks2023instructpix2pix, cao2023masactrl, koo2024flexiedit, yoon2024dni} because it must preserve geometric coherence and consistency over the whole shape. Instant3dit~\cite{barda2025instant3dit} formulates 3D editing as multiview image inpainting followed by reconstruction, while more recent methods directly edit native 3D latent spaces. VoxHammer~\cite{li2025voxhammer} improves local editing by reusing inverted latents and cached features for preserved regions, Nano3D~\cite{Ye2025Nano3D} extends FlowEdit~\cite{kulikov2025flowedit} into TRELLIS~\cite{xiang2025structured} and introduces connectivity-aware region merging, and AnchorFlow~\cite{zhou2026anchorflow} stabilizes inversion-free editing through latent-anchor consistency under timestep-dependent noise. Compared with these methods, our approach focuses on explicit per-token actuation in the tangent space, directly exploiting the spatial locality of VecSet tokens for localized editing.

\section{Preliminaries}

\subsection{VecSet-based 3D Generative Model}
We consider a pretrained flow-based 3D generative model defined over a VecSet latent representation. Given a condition $c$ (e.g., a text prompt or a reference image), the model operates on a latent state
\begin{equation}
X_t = \{x_t^1, x_t^2, \dots, x_t^N\}, \qquad x_t^i \in \mathbb{R}^d,
\end{equation}
where each $x_t^i$ is a latent token and $N$ is the number of tokens. In our setting, the pretrained Hunyuan3D 2.1 model uses $N=4096$ tokens. The model predicts a conditional velocity field $v_\theta(X_t, c, t)$, which defines the generative dynamics in latent space. Starting from a noisy latent state, the model follows this velocity field to produce a clean latent representation, which is then decoded into a 3D mesh. Although VecSet represents a 3D object as a set of latent tokens, it has often been regarded as a relatively global representation because the tokens interact through self-attention and share global contextual information. In this view, perturbing one token may influence others through global token interactions, making VecSet-based models appear less suitable for localized editing than voxel-based representations.

\subsection{Training-free Editing in Flow-based Models}
In training-free editing, we are given a source 3D object and a target condition describing the desired edit, and the goal is to obtain an edited 3D mesh without finetuning the pretrained generator. Let $c_{src}$ and $c_{tar}$ denote the source and target conditions. At the current editing state $X_t^{FE}$, the model predicts two conditional velocity fields,
\begin{equation}
v_{src} = v_\theta(X_t^{FE}, c_{src}, t), \qquad
v_{tar} = v_\theta(X_t^{FE}, c_{tar}, t).
\end{equation}
Most existing methods derive an editing signal from their difference, $\Delta v = v_{tar} - v_{src}$, and use it to update the editing trajectory. Although $\Delta v$ is already defined token-wise because both $v_{src}$ and $v_{tar}$ inherit the tokenized VecSet structure, prior methods do not explicitly actuate it at the token level. Instead, they typically apply a globally shared scaling factor to the full velocity difference, so preserved and editable regions are often perturbed together.

\begin{figure*}[t]
  \begin{center}
    \centerline{\includegraphics[width=\textwidth]{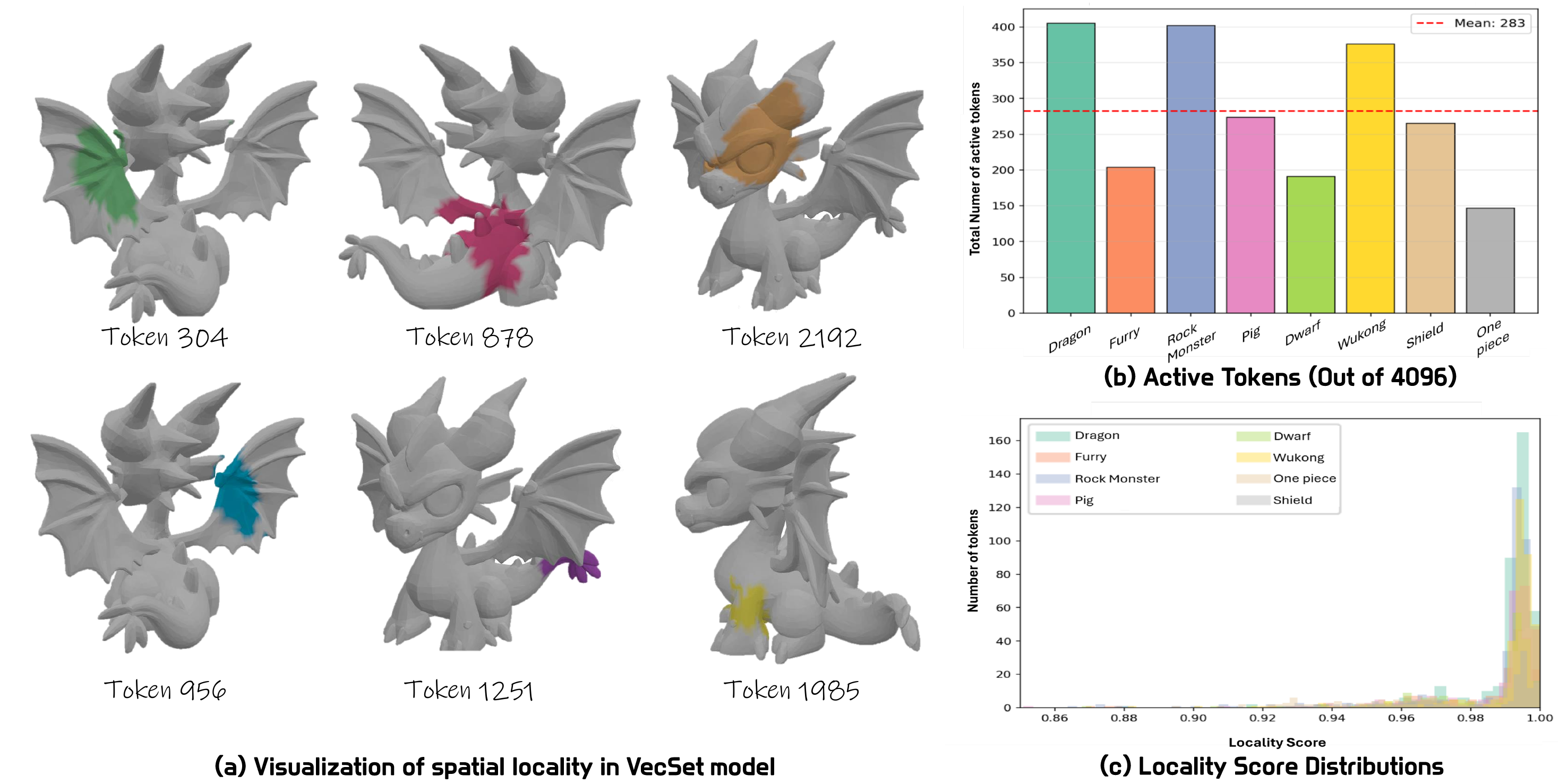}}
    \caption{(a) visualizes the localized influence of individual tokens in the VecSet model, where the colored areas indicate specific regions activated by each respective token. (b) illustrates the number of active tokens out of a total of 4096 triggered during the inference of each data category shown on the x-axis. (c) presents the locality score distribution for datasets across various categories, quantifying how locally the VecSet model responds. The y-axis represents the number of activated tokens corresponding to each locality score.}
    \label{fig:fig4}
  \end{center}
\end{figure*}

\subsection{Observation: Token-wise Spatial Locality in VecSet}
\label{sec:3.3}
Despite reports that transferring trajectory-based editing to VecSet-based 3D generators can be unstable potentially due to their more global token interactions~\cite{Ye2025Nano3D}, our empirical analysis shows that tokens in the VecSet-based Hunyuan3D 2.1 model often exhibit strong spatial locality in practice\footnote{We also demonstrate in the Appendix that this strong token locality holds true across other generative models trained with the VecSet representation.}. As illustrated in Fig.~\ref{fig:fig4}(a), visualizing individual tokens reveals that their geometric influence is typically concentrated on a specific part (e.g., wing, tail, arm, or facial region) rather than being distributed across the entire object. This suggests that each token tends to predominantly affect a limited geometric component, while other regions are shaped by different tokens.

\begin{figure*}[t!]
  \begin{center}
    \centerline{\includegraphics[width=\textwidth]{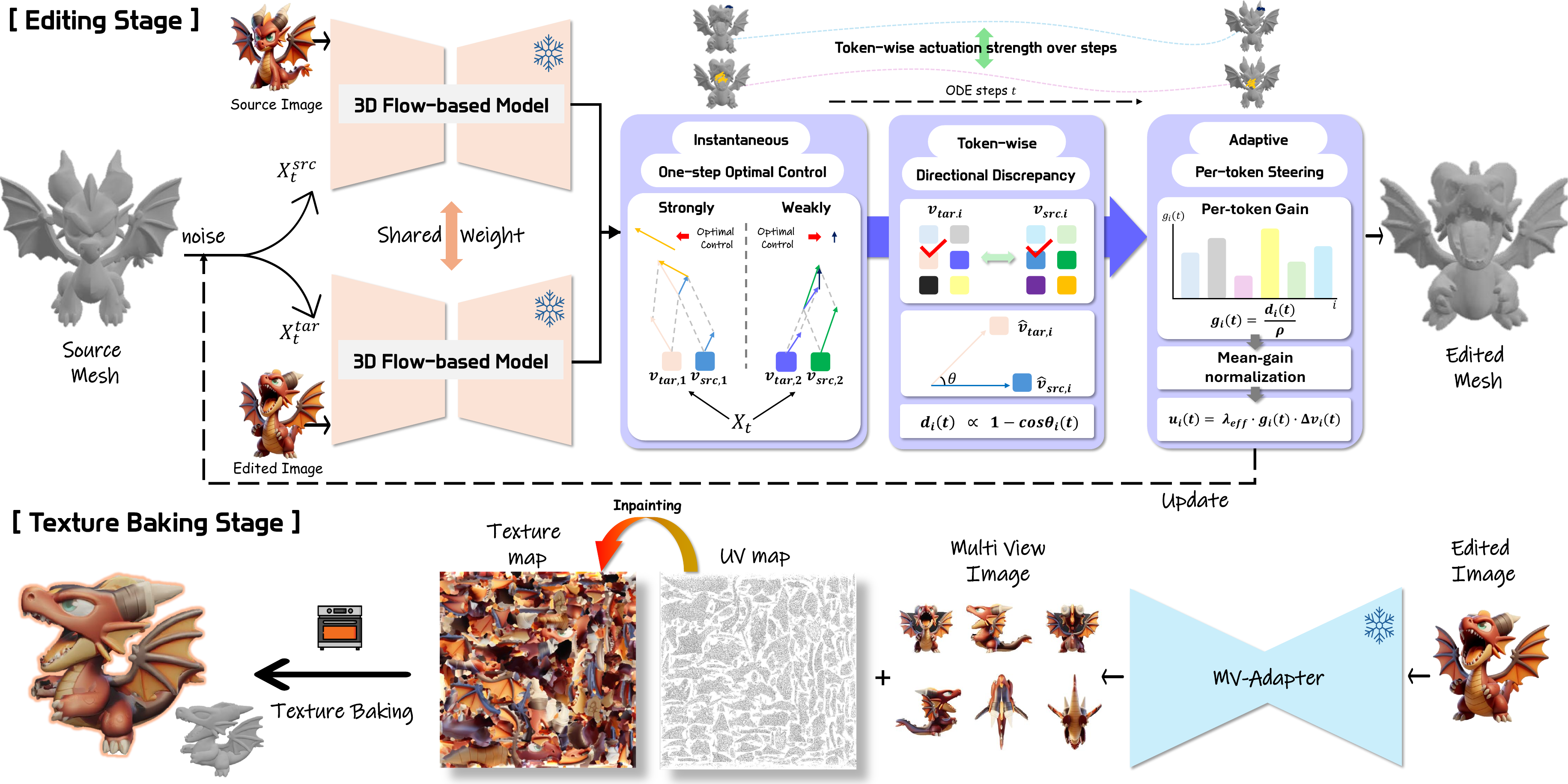}}
    \caption{\textbf{Overview of TanGO.}
Editing stage: At each ODE step, we compute token-wise source/target velocities and their difference $\Delta v_i(t)$.
To selectively amplify update signals in editable regions while suppressing those in regions to be preserved, we formulate an instantaneous one-step optimal control that yields a per-token steering gain and applies $u_i(t)$.
The demand $d_i(t)$ is defined by the token-wise directional discrepancy between normalized velocities, and mean-gain normalization keeps the overall guidance energy consistent across steps.
Texture baking stage: We render multi-view images via an MV-Adapter, construct UV/texture maps, inpaint missing texels, and bake the texture onto the edited mesh.}
    \label{fig:fig5}
  \end{center}
\end{figure*}

This locality is also supported quantitatively. Across multiple object categories, we find that only a subset of the 4096 latent tokens becomes active for a given object. As shown in Fig.~\ref{fig:fig4}(b), the number of active tokens remains well below the full token budget, indicating that generation in VecSet space is sparse rather than uniformly dense.

To quantify locality more directly, we compute a locality score to measure the spatial concentration of each token's geometric influence.
As illustrated in Fig.~\ref{fig:fig4}(c), the majority of tokens achieve high locality scores (e.g., above 0.96) across datasets. Therefore, while self-attention enables global token interactions, the geometric effect of individual tokens remains highly localized in our model.

This observation has an important implication for training-free 3D editing: if different tokens correspond to different local geometric regions, then editing should not apply a uniform control signal to all tokens. Instead, control should be selectively strengthened for tokens in modified areas while being suppressed for tokens in preserved regions. Motivated by this, we propose \textbf{TanGO}, which derives token-wise control signals directly from the source and target velocity fields to perform localized, geometry-aware steering without explicit spatial masks.

\section{Method}
Existing training-free 3D editing methods steer the latent trajectory using the source--target velocity difference
$\Delta v = v_{\mathrm{tar}} - v_{\mathrm{src}}$, but typically apply a \emph{token-shared} scaling schedule.
As a result, editable and preserved regions are often perturbed together.
Motivated by the token-wise spatial locality of VecSet latents (sec.~\ref{sec:3.3}), we cast editing as \emph{tangent-space control} and decouple two coupled questions at every ODE step:
\textbf{(i) where to actuate} (localization across tokens) and \textbf{(ii) how strongly to actuate} (trajectory-level guidance). TanGO addresses this by casting editing as a per-token \emph{allocation} problem: a closed-form minimal-intervention control law is fully determined once we specify a principled token-wise demand $d_i(t)$, while a trajectory-level normalization stabilizes the overall guidance energy (Fig.~\ref{fig:fig5}). Additionally, to overcome the absence of a native texture pipeline in VecSet models, we integrate a process that generates and bakes texture maps using multi-view diffusion\cite{huang2025mv} and inpainting models\cite{suvorov2022resolution}.

\subsection{Instantaneous One-step Optimal Control}
\label{sec:one_step_oc}
We perturb the source dynamics by a token-wise control input $u_i(t)$, yielding
\begin{equation}
\dot x_t^i = v_{\mathrm{src},i}(t) + u_i(t).
\end{equation}
At each ODE step, we choose $u_i(t)$ to maximize first-order progress toward the target tangent direction $\Delta v_i(t)$,
while penalizing excessive intervention with a quadratic control-energy regularizer:
\begin{equation}
\max_{u_i}\;
d_i(t)\,\langle u_i,\Delta v_i(t)\rangle
-\frac{\rho}{2}\|u_i\|^2
\quad
\text{s.t.}\quad
u_i \in \mathrm{span}(\Delta v_i(t)),
\label{eq:one_step_oc}
\end{equation}
where $d_i(t)\ge 0$ is the token-wise demand and $\rho>0$ penalizes control energy.
The span constraint enforces a minimal-intervention principle: it restricts updates to the only direction that directly
transfers the source flow toward the target flow, thereby avoiding orthogonal perturbations that may harm preservation.
Parameterizing $u_i(t)=g_i(t)\Delta v_i(t)$ reduces Eq.~\eqref{eq:one_step_oc} to a concave quadratic in $g_i(t)$, which admits the
closed-form solution
\begin{equation}
g_i^*(t)=\frac{d_i(t)}{\rho},
\qquad
u_i^*(t)=\frac{d_i(t)}{\rho}\Delta v_i(t).
\label{eq:ui_star_compact}
\end{equation}
Eq.~\eqref{eq:ui_star_compact} shows that TanGO reduces editing to a per-token \emph{allocation} rule: each token receives a gain
proportional to its demand $d_i(t)$, while $\Delta v_i(t)$ preserves the full transformation direction.
Thus, the entire control law is determined once we specify a principled demand $d_i(t)$.
In the next section, we design $d_i(t)$ to be (i) invariant to velocity magnitude, (ii) bounded to prevent spurious spikes,
and (iii) informative of whether token $i$ already follows the target-conditioned tangent direction.

\subsection{Token-wise Directional Discrepancy}
\label{sec:directional_discrepancy}

We define $d_i(t)$ from a probabilistic model of \emph{directional} consistency between the source and target-conditioned
velocity predictions. Since velocity magnitudes vary across timesteps and assets, we use directional agreement as a
scale-stable indicator of whether token $i$ already follows the target-conditioned flow.
We normalize token-wise velocities and compute their cosine similarity as
\begin{equation}
\hat v_{\mathrm{src},i}(t)=\frac{v_{\mathrm{src},i}(t)}{\|v_{\mathrm{src},i}(t)\|+\varepsilon},\qquad
\hat v_{\mathrm{tar},i}(t)=\frac{v_{\mathrm{tar},i}(t)}{\|v_{\mathrm{tar},i}(t)\|+\varepsilon},
\end{equation}
\begin{equation}
\cos\theta_i(t)=\hat v_{\mathrm{src},i}(t)^\top \hat v_{\mathrm{tar},i}(t).
\end{equation}
Because $\hat v_{\mathrm{src},i},\hat v_{\mathrm{tar},i}$ lie on the unit hypersphere, we model $\hat v_{\mathrm{tar},i}$ as
directional data concentrated around mean direction $\hat v_{\mathrm{src},i}$ using the von Mises-Fisher (vMF) likelihood:
\begin{equation}
p(x\mid \mu,\kappa)=C_d(\kappa)\exp\!\big(\kappa\,\mu^\top x\big),
\qquad
x=\hat v_{\mathrm{tar},i}(t),~\mu=\hat v_{\mathrm{src},i}(t).
\end{equation}
Up to an additive constant independent of $x$, the negative log-likelihood satisfies
\begin{equation}
-\log p(\hat v_{\mathrm{tar},i}\mid \hat v_{\mathrm{src},i},\kappa) = \text{const}-\kappa\cos\theta_i(t)
~\equiv~ \text{const}+\kappa(1-\cos\theta_i(t)).
\end{equation}
We therefore define the demand as the normalized directional negative log-likelihood:
\begin{equation}
d_i(t)=1-\cos\theta_i(t),
\label{eq:di_def}
\end{equation}
and absorb the concentration to energy ratio $\kappa/\rho$ into the global scale $\lambda$ used later.
This yields a bounded and geometry-consistent signal: $d_i(t)\in[0,2]$ and
$\|\hat v_{\mathrm{src},i}-\hat v_{\mathrm{tar},i}\|^2=2\,d_i(t)$.
Hence, aligned directions receive near-zero demand, while larger mismatch implies stronger actuation.
Eq.~\eqref{eq:di_def} corresponds to the (scaled) vMF negative log-likelihood over directions after removing constants
independent of $\hat v_{\mathrm{tar},i}$. \footnote{We provide the detailed vMF derivation and related geometric identities in the Appendix.}

\begin{figure*}[t]
  \begin{center}
    \centerline{\includegraphics[width=\textwidth]{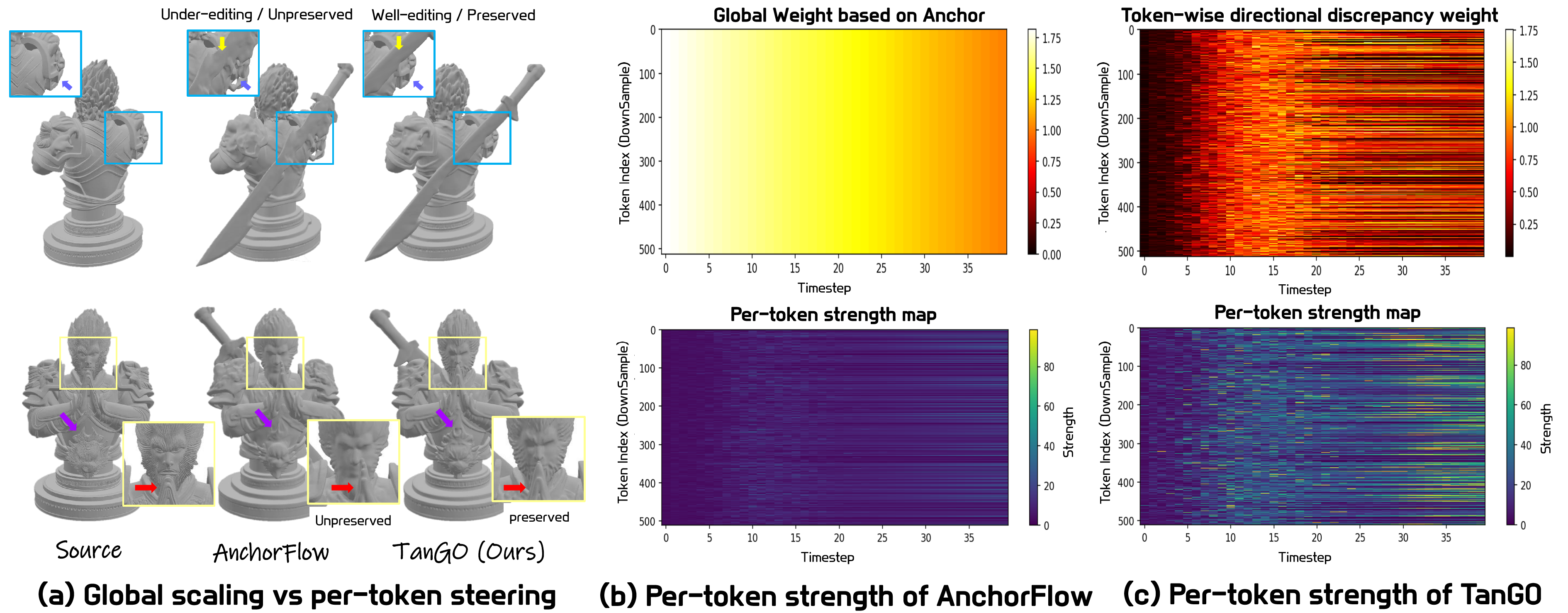}}
    \caption{(a) Comparison between AnchorFlow (global scaling) and TanGO (per-token steering). AnchorFlow exhibits under-editing and unpreserved regions, whereas TanGO yields well-edited results. (b) Token-wise contribution of AnchorFlow across timesteps $t$, shown via a global weight map and per-token strength map. (c) Per-token strength of TanGO, demonstrating its ability to adaptively assign stronger signals to target tokens and weaker signals to regions to be preserved.}
    \label{fig:fig6}
  \end{center}
\end{figure*}

\subsection{Adaptive Per-token Steering}
\label{sec:adaptive_steering}

Given $d_i(t)$, we define the raw per-token gain $g_i(t)=\frac{d_i(t)}{\rho}$. This decouples \emph{where} to edit from \emph{how strongly} to edit: $g_i(t)$ allocates actuation across tokens based on
directional mismatch (thus insensitive to magnitude), while $\Delta v_i(t)$ retains the full direction and magnitude needed to
execute the transformation.
However, the mean gain $\bar g(t)=\frac{1}{N}\sum_i g_i(t)$ can fluctuate over time, inducing unintended variation in the
overall guidance energy.
We stabilize global edit strength via mean-gain normalization:
\begin{equation}
\lambda_{\mathrm{eff}}(t)=\lambda \cdot \frac{\eta}{\bar g(t)+\varepsilon},
\label{eq:lambda_eff}
\end{equation}
where $\eta$ is a target energy level and $\varepsilon$ prevents division by zero.
The final token-wise control and dynamics are

\begin{align}
u_i(t) &= \lambda_{\mathrm{eff}}(t)\, g_i(t)\, \Delta v_i(t), \label{eq:optimal_u}\\
\dot x_t^i &= v_{\mathrm{src},i}(t)+u_i(t). \label{eq:final_dynamics}
\end{align}

Fig.~\ref{fig:fig6} visualizes this decoupling.
Unlike a global scaling that offers limited token selectivity, TanGO produces a non-uniform
$d_i(t)$ and concentrates effective steering on high-demand tokens while suppressing changes on tokens corresponding to regions to be
preserved, enabling localized, mask-free edits.

\section{Experiment}
\subsection{Implementation Details}
\paragraph{\textbf{Settings}} The editing process integrates over $T = 50$ timesteps from $n_{\max} = 41$ down to $n_{\min} = 1$. We set the classifier-free guidance scales to $s_{\text{src}} = 3.5$ and $s_{\text{tar}} = 7.5$, with additional default parameters $\lambda = 5.0$ and $\eta = 0.2$. The final edited latent $X_0^{\text{edit}}$ is directly decoded into 3D space. All experiments are conducted on NVIDIA A100 GPUs.

\paragraph{\textbf{TanGOEdit Datasets}} To systematically evaluate the capability of training-free 3D editing and topology preservation, we construct a comprehensive benchmark named TanGOEdit.\footnote{A detailed description of the dataset is provided in the Appendix.}
Comprising exactly 100 high-quality editing samples, our benchmark is built by carefully selecting 3D assets with high aesthetic scores from the Objaverse-XL \cite{deitke2023objaverse}, TRELLIS-500K \cite{xiang2025structured}, and Google Scanned Objects (GSO) \cite{downs2022google} datasets.
Inspired by Nano3D \cite{Ye2025Nano3D}, we utilize Qwen 2.5 VL \cite{bai2025qwen25vltechnicalreport} to generate diverse editing instructions.

\paragraph{\textbf{Metrics.}} To comprehensively assess the editing efficacy, we utilize three complementary metrics. Semantic alignment is evaluated via $\text{CLIP-T}$~\cite{radford2021learning}, which calculates the correspondence between the rendered views of the edited 3D geometry and the target text prompt. For structural and identity preservation, we compute $\text{CLIP-I}$~\cite{radford2021learning} and $\text{DINO-I}$~\cite{caron2021emerging} to measure the visual and feature-level similarity between the edited renderings and the source condition image.

\begin{figure*}[t!]
  \begin{center}
    \centerline{\includegraphics[width=\textwidth]{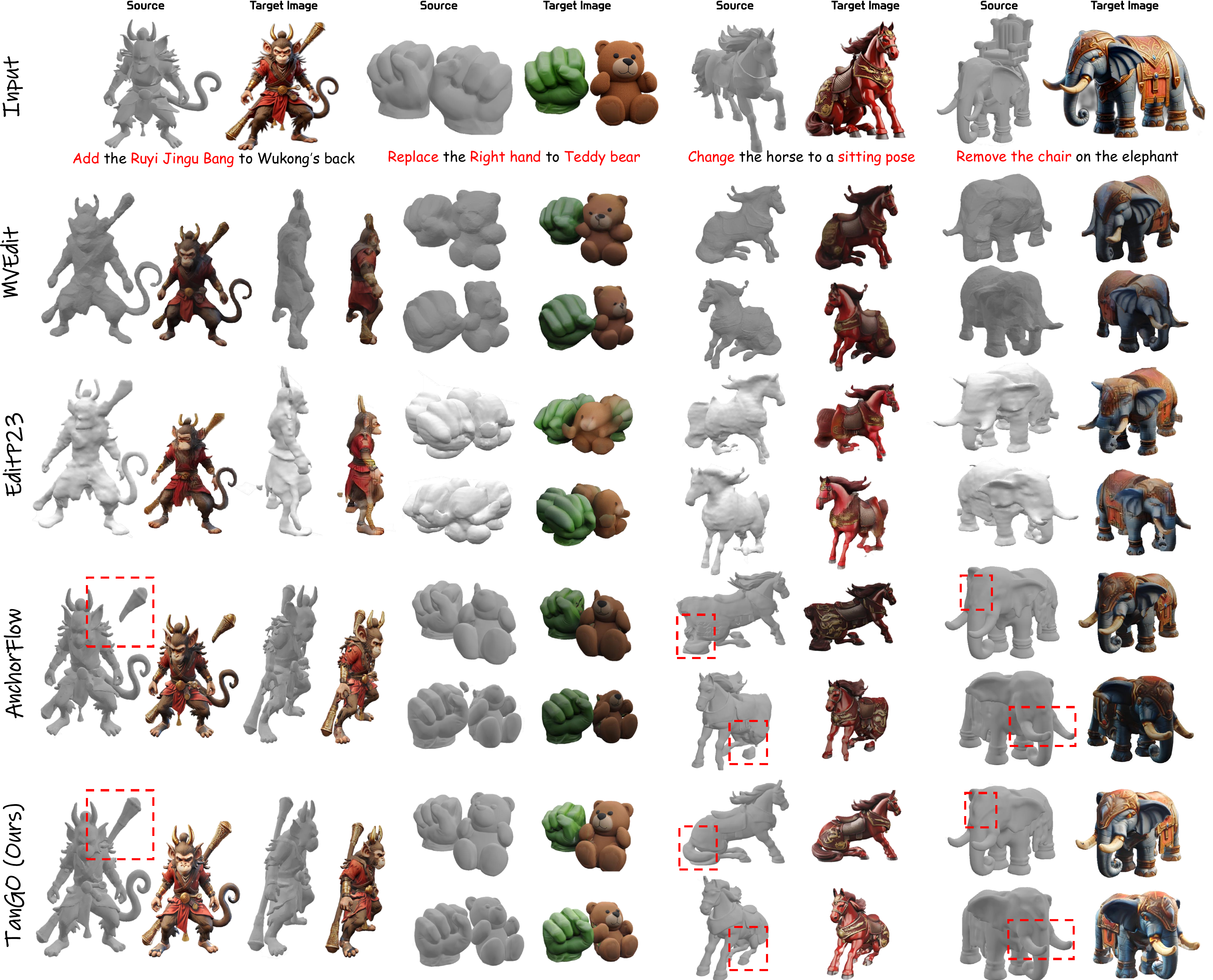}}
    \caption{ \textbf{Qualitative comparisons on text-driven 3D editing.} Compared to MVEdit, EditP23, and AnchorFlow, TanGO achieves highly localized edits across diverse tasks while preserving regions that should remain unchanged, effectively avoiding unintended artifacts (see the red boxes).}
    \label{fig:fig7}
  \end{center}
\end{figure*}
\subsection{Qualitative Results}
Fig.~\ref{fig:fig1},~\ref{fig:fig7} presents qualitative comparisons using our newly constructed TanGOEdit dataset across diverse text-driven 3D editing tasks, including object addition, part replacement, pose change, and object removal. Compared to baseline methods, TanGO achieves highly localized edits while strictly preserving the geometry in regions that are not intended to change.
In particular, for rigid editing tasks such as addition, replacement, and removal, baseline methods often suffer from under-editing due to limited editing strength, or they propagate edits to adjacent areas, leading to undesired deformations or the loss of fine structures. These failures often appear as boundary bleeding or collateral distortions in nearby parts (e.g., unintended warping or collapse), even when the requested edit is spatially confined. Conversely, TanGO successfully modifies the target elements without introducing noticeable distortions to adjacent body parts.
For non-rigid editing tasks, such as pose changes, our method better maintains global shape coherence and prevents the mesh from breaking or stretching. This indicates that the per-token steering mechanism operates stably even under strong transformations. Importantly, this stability is where naive global guidance often destabilizes the trajectory and amplifies artifacts.
Overall, these results demonstrate that TanGO can successfully control and decouple edit strength from source preservation, yielding stable and geometry-consistent edits without the need for external masks or additional training.

\captionsetup[table]{position=top, skip=6pt}

\sisetup{
  detect-weight=true,
  detect-inline-weight=math,
  table-number-alignment=center
}

\newcommand{\NA}{\multicolumn{1}{c}{--}}

\begin{table*}[t]
\centering

\begin{minipage}[t]{0.48\textwidth}
\centering
\caption{\textbf{Quantitative comparison on TanGOEdit.}
Our proposed TanGO consistently achieves the highest scores across all metrics: DINO-I, CLIP-I, and CLIP-T, demonstrating superior semantic alignment and generation fidelity compared to state-of-the-art 3D editing baselines.}
\label{tab:quant_comp}
\scriptsize
\begin{tabular}{l
S[table-format=1.4]
S[table-format=1.4]
S[table-format=1.4]}
\toprule
\textbf{Method} & {DINO-I$\uparrow$} & {CLIP-I$\uparrow$} & {CLIP-T$\uparrow$} \\
\midrule
MVEdit ~\cite{chen2024generic}      & 0.5861 & 0.4994   & 0.1289    \\
EditP23 ~\cite{bar2025editp23}     & 0.5838 & 0.4971   & 0.1313    \\
FlowEdit ~\cite{kulikov2025flowedit}    & 0.6127 & 0.7205 & 0.2081 \\
AnchorFlow ~\cite{zhou2026anchorflow}  & 0.6426 & 0.7492   & 0.2124    \\
\midrule
\textbf{TanGO(Ours)} & \textbf{0.6510} & \textbf{0.7679} & \textbf{0.2301} \\
\bottomrule
\end{tabular}
\end{minipage}
\hfill
\begin{minipage}[t]{0.48\textwidth}
\centering
\caption{\textbf{Results of the User Study.}
Participants strongly preferred TanGO over prior methods across all criteria. Our method significantly outperforms AnchorFlow, particularly in shape preservation (87\%) and visual quality (84\%), confirming its effectiveness in practical editing scenarios.}
\label{tab:user_study}
\scriptsize
\begin{tabular}{l c c c}
\toprule
\textbf{Method} & \textbf{Prompt} & \textbf{Visual} & \textbf{Shape} \\
                & \textbf{Align.}  & \textbf{Quality} & \textbf{Preserv.} \\
\midrule
AnchorFlow ~\cite{zhou2026anchorflow} & 21\% & 16\% & 13\% \\
TanGO(Ours) & 79\% & 84\% & 87\% \\
\bottomrule
\end{tabular}
\end{minipage}

\end{table*}

\subsection{Quantitative Results}

In this section, we compare \textbf{TanGO} with existing 3D editing methods, including MVEdit, EditP23, FlowEdit, and AnchorFlow. As shown in Table~\ref{tab:quant_comp}, TanGO achieves the best performance on the main TanGOEdit benchmark, with DINO-I, CLIP-I, and CLIP-T scores of 0.6510, 0.7679, and 0.2301, respectively. The improvement in CLIP-T indicates stronger alignment with the target editing instruction, while the gains in DINO-I and CLIP-I suggest better preservation of the source identity and visual structure. These results support our claim that adaptive per-token steering provides a more favorable balance between semantic editability and source preservation than prior global trajectory steering methods.

The user study in Table~\ref{tab:user_study} further supports this trend. Participants preferred TanGO over AnchorFlow in prompt alignment, visual quality, and shape preservation, with preference rates of 79\%, 84\%, and 87\%, respectively. The largest margin is observed in shape preservation, indicating that TanGO more effectively suppresses unintended changes in preserved regions while maintaining the requested semantic edit.

We additionally conduct a mask-aware 3D evaluation on a separate 150-sample benchmark, consisting of 100 samples from Edit3DBench~\cite{li2025voxhammer} and 50 TanGOEdit samples annotated in Blender. This separate protocol is necessary because 3D geometric preservation metrics require ground-truth edit masks, which are not available for the full TanGOEdit benchmark. Under this setting, TanGO also outperforms MVEdit, EditP23, FlowEdit, AnchorFlow, VoxHammer, and Nano3D on 3D-native text--shape alignment and mask-aware geometric preservation metrics, including Uni3D$_{\mathrm{txt}}$, CD$_{\mathrm{preserve}}$, and NC$_{\mathrm{preserve}}$. Detailed quantitative results and qualitative comparisons are provided in the supplementary material.

\begin{figure*}[t!]
  \begin{center}
    \centerline{\includegraphics[width=\textwidth]{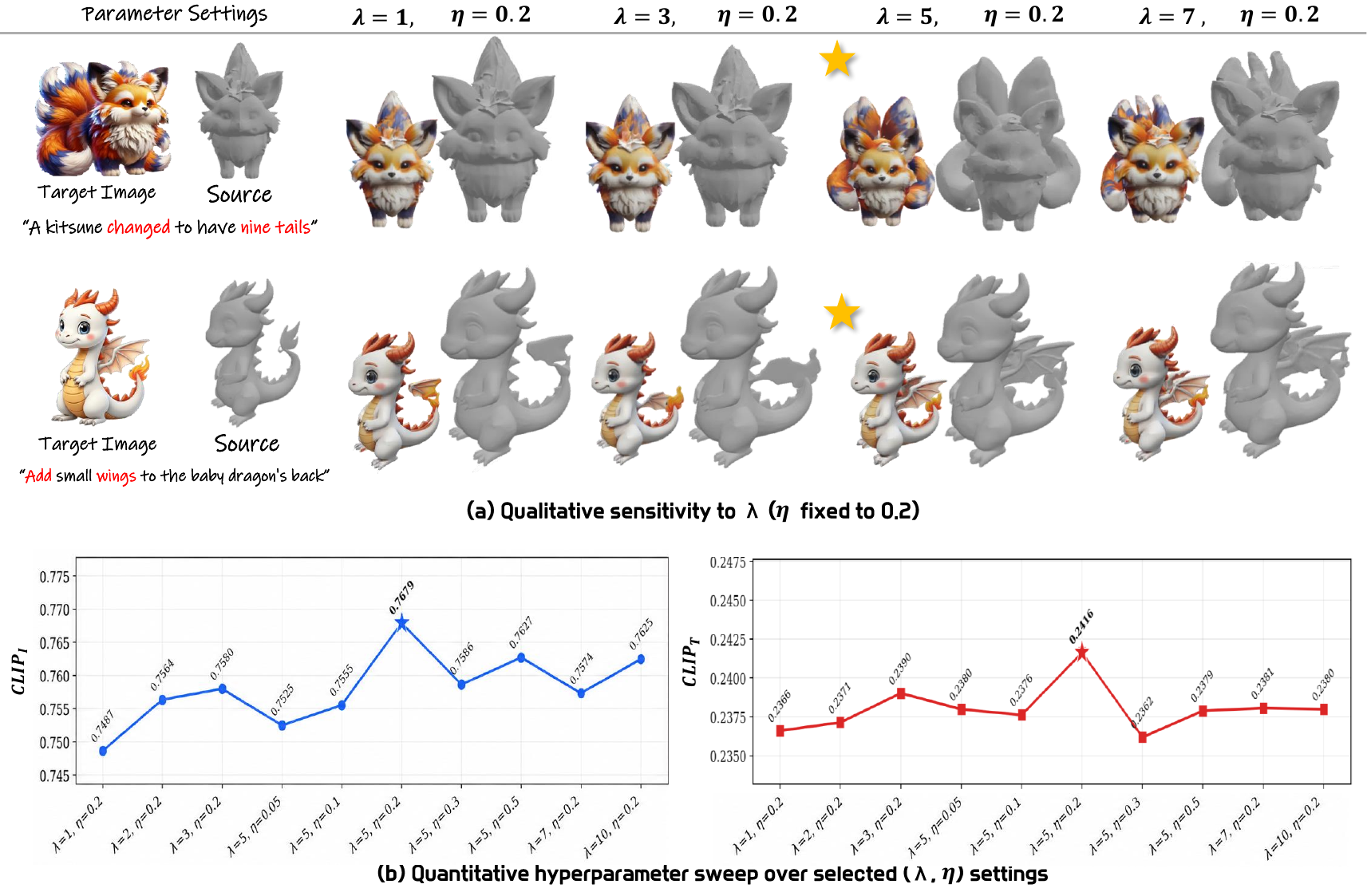}}
    \caption{ \textbf{Qualitative and Quantitative Analysis of Parameter Selection} (a) Qualitative comparison under varying $\lambda$ settings ($\eta=0.2$). The optimal setting $\lambda=5$ (marked with a star) achieves clear geometric modifications, such as the distinct nine tails and wings, while preserving source integrity.
    (b) Quantitative hyperparameter sweep on \textbf{TanGOEdit} datasets. The CLIP-I and CLIP-T scores peak at $\lambda=5, \eta=0.2$, justifying our choice of optimal default parameters.
    }
    \label{fig:fig8}
    \vspace{-0.8cm}
  \end{center}
\end{figure*}

\subsection{Ablation Study}
\paragraph{\textbf{Hyperparameter sensitivity}} We study the effect of hyperparameters on performance by presenting both qualitative and quantitative analyses.
We first examine how varying $\lambda$ changes the trade-off between edit strength and preservation, and then report CLIP-based alignment scores over a range of $(\lambda,\eta)$ settings.

Fig.~\ref{fig:fig8}(a) shows qualitative results when increasing $\lambda$ while fixing $\eta{=}0.2$.
As $\lambda$ grows, edits become stronger, but overly large values may start to degrade preservation.
Across the tested examples, we observe that our default setting provides a favorable balance between effective editing and geometric preservation.

Fig.~\ref{fig:fig8}(b) visualizes CLIP-based alignment scores (CLIP-I and CLIP-T) evaluated on multiple $(\lambda,\eta)$ combinations.
Specifically, we vary $\lambda \in \{1,2,3,5,7,10\}$ and combine it with $\eta \in \{0.05, 0.1,0.2,0.3,0.5\}$.
The scores change smoothly across settings rather than exhibiting sharp spikes or drops, and the default choice $(\lambda{=}5, \eta{=}0.2)$ lies near a high-performing region.
Overall, these results suggest that TanGO behaves stably under moderate variations of $(\lambda,\eta)$.

\begin{figure*}[t]
  \begin{center}
    \centerline{\includegraphics[width=\textwidth]{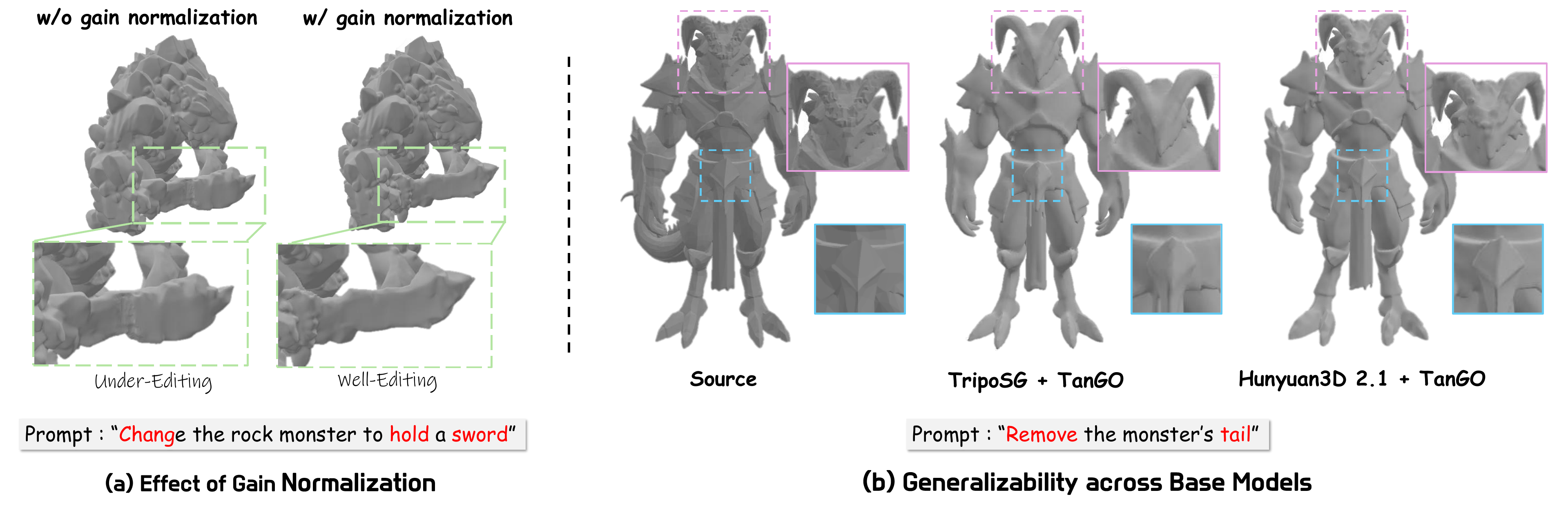}}
    \caption{\textbf{Qualitative Analysis of TanGO Components.}
    (a) Gain Normalization prevents under-editing, ensuring accurate geometric changes (e.g., holding a sword). (b) TanGO shows high generalizability across various base models like TripoSG and Hunyuan3D 2.1.}
    \label{fig:fig9}
  \end{center}
\end{figure*}

\paragraph{\textbf{Effect of gain normalization}}
To verify the necessity of the proposed gain normalization, we conduct an ablation study by replacing the timestep dependent scaling factor $\lambda_{\mathrm{eff}}(t)=\lambda\cdot \frac{\eta}{\bar g(t)+\varepsilon}$ with a fixed constant $\lambda$. This experiment aims to analyze how the mechanism that compensates for the mean gain $\bar g(t)$ which fluctuates significantly along the integration trajectory affects the final editing quality.

Our results show that without gain normalization, it becomes difficult to maintain a consistent steering strength throughout the entire editing process. Even when the relative weights for each token remain the same, the lack of normalization causes the guidance to be too weak at certain timesteps, leading to \textit{under-editing} where the target changes are not fully realized. Conversely, at other stages, the guidance can become overly aggressive, potentially harming the structural stability of the mesh.

Fig.~\ref{fig:fig9}(a) visually demonstrates this effect through a representative example. When $\lambda_{\mathrm{eff}}$ is set to a fixed value, the model lacks sufficient power to implement the target shape. In contrast, the full TanGO model with gain normalization provides a balanced guide across all timesteps. This allows for strong and stable edits while preserving the structural integrity of the surrounding geometry. Ultimately, gain normalization ensures that the global edit strength is optimized without changing the per-token steering patterns, making the editing process more robust even during significant shape transformations.

\paragraph{\textbf{Generalizability of TanGO}}
To evaluate whether TanGO generalizes beyond Hunyuan3D 2.1, we apply the same token-wise steering mechanism to TripoSG, another VecSet-based 3D foundation model. As shown in Fig.~\ref{fig:fig9}(b), TanGO successfully performs the requested semantic edit without additional training. Quantitative results further show that TanGO improves over AnchorFlow under both 2D alignment metrics and mask-aware 3D preservation metrics, supporting its applicability beyond a single backbone. We observe that TripoSG results exhibit slightly smoother geometry than Hunyuan3D 2.1, which is mainly due to the reconstruction capacity of the underlying 3D VAE rather than the TanGO optimization process itself. Detailed cross-model results are provided in the supplementary material.

\paragraph{\textbf{Ablation on Per-token Gain Design}}
To verify that the proposed directional demand is not merely an arbitrary token-wise rescaling, we compare TanGO with simpler gain designs, including global scaling, L2-magnitude-based gain $|v_{\mathrm{tar},i}-v_{\mathrm{src},i}|$, and inner-product-based gain $\langle v_{\mathrm{src},i}, v_{\mathrm{tar},i}\rangle$. TanGO consistently performs better than these alternatives, supporting the use of a bounded and scale-invariant directional discrepancy rather than raw velocity magnitude or unnormalized similarity. Full results are provided in the supplementary material.

\begin{table}[t!]
\centering
\caption{Comparison of inference time across different 3D editing frameworks.}
\label{tab:time_comparison}
\scriptsize
\begin{tabular}{lccccc}
\toprule
 & MVEdit & EditP23 & FlowEdit & AnchorFlow & \textbf{TanGO (Ours)} \\
\midrule
Time (s) & 632.45 & 59.32 & 28.13 & 28.76 & 28.73 \\
\bottomrule
\end{tabular}
\end{table}

\paragraph{\textbf{Comparison of Time Cost}}
As shown in Table~\ref{tab:time_comparison}, TanGO achieves an inference time of 28.73s per edit, which is comparable to FlowEdit (28.13s) and AnchorFlow (28.76s), while being much faster than MVEdit (632.45s) and EditP23 (59.32s). This indicates that the proposed per-token steering and gain normalization introduce negligible computational overhead despite adding token-level control to the editing process. As a result, TanGO maintains the efficiency of recent flow-based editing methods while providing more localized and preservation-aware edits, making it practical for iterative 3D design workflows.

\section{Conclusion}
In this paper, we introduced TanGO, a novel, training-free framework for 3D mesh editing that leverages the rich priors of pre-trained flow-based 3D foundation models. By formulating the editing process as an instantaneous one-step optimal control problem, TanGO achieves precise and localized semantic modifications. Our core mechanism, adaptive per-token steering, dynamically scales the control input based on directional discrepancy in the tangent space. This allows for stable editing in targeted regions while preserving the geometry of unedited areas, all without the need for costly fine-tuning or model retraining. While our extensive evaluations demonstrate TanGO's robust generalizability across different network architectures, we also identify that the ultimate preservation of high-frequency geometric details is currently bounded by the representational capacity of the underlying model's 3D VAE. Overcoming this VAE bottleneck potentially through the integration of more expressive latent representations or hybrid pixel-3D refinement strategies presents a promising avenue for future work. Ultimately, TanGO offers a highly efficient and adaptable solution, advancing the practical utility of 3D foundation models for high-fidelity, user-guided 3D creation.

\section*{Acknowledgements}
This work was supported by Institute for Information \& communications Technology Planning \& Evaluation (IITP) grant funded by the Korea government (MSIT) (No.RS-2021-II211381, Development of Causal AI through Video Understanding and Reinforcement Learning, and Its Applications to Real Environments) and (No.RS-2022-II220184, Development and Study of AI Technologies to Inexpensively Conform to Evolving Policy on Ethics).

%
%
\bibliographystyle{splncs04}
\bibliography{main}

\clearpage

\title{Supplementary Material for \\ TanGO: Training-Free 3D Editing via Tangent-space Guidance and Optimization}

\titlerunning{TanGO}

\author{Siwoo Lim\inst{1}\orcidlink{0009-0001-1187-6548} \and
Sunjae Yoon\inst{2}\orcidlink{0000-0001-7458-5273} \and
Gwanhyeong Koo\inst{1}\orcidlink{0009-0005-6455-3223} \and
Hyeonseo Yun\inst{1}\orcidlink{0009-0001-2486-1142} \and
Chang D. Yoo\inst{1}\thanks{Corresponding Author}\orcidlink{0000-0002-0756-7179}}

\authorrunning{S. Lim et al.}

\institute{Korea Advanced Institute of Science and Technology, Daejeon, Republic of Korea \\
\email{\{siu4450, kookie, hyunseo, cd\_yoo\}@kaist.ac.kr} 
\and
Chung-Ang University, Seoul, Republic of Korea\\
\email{\{sunjaeyoon\}@cau.ac.kr}}

\maketitle


\appendix
\inappendixtoctrue
\supplementarytoc
\clearpage

\section{Spatial Locality of VecSet Tokens Across Different 3D Generative Models}
\label{app:locality_across_models}

In the main paper, we showed that VecSet tokens in Hunyuan3D exhibit strong spatial locality, i.e., each token primarily responds to a limited spatial region of the decoded 3D asset. Here, we further examine whether this phenomenon is specific to a single architecture or persists across different generative models trained with the VecSet representation.

To this end, we compare Hunyuan3D (4096 tokens) with TripoSG (2048 tokens), which adopts a different architecture and token budget. Although the two models show substantially different token activation patterns, both exhibit consistently high spatial locality. This provides additional evidence that token-wise spatial locality is not tied to a particular sparsity regime, but persists across different VecSet-based 3D generative models.

\subsection{Attention-based Spatial Locality Measure}
\label{app:locality_metric}

To quantitatively compare token-wise spatial locality across models, we use an attention-based proxy derived from the decoder cross-attention over sampled surface points. Intuitively, a token is considered spatially localized if the surface points that strongly attend to that token are concentrated within a compact region of the 3D shape.

\paragraph{\textbf{Surface sampling.}}
Given a decoded mesh, we uniformly sample \(M=30{,}000\) points from the mesh surface. Let the sampled points be denoted by
\begin{equation}
\mathcal{P}=\{p_i\}_{i=1}^M, \qquad p_i \in \mathbb{R}^3.
\end{equation}

\paragraph{\textbf{Cross-attention capture.}}
For each sampled surface point \(p_i\), we extract the cross-attention weights from the last decoder block. We record the top-\(K\) tokens with the highest attention values for that point, where \(K=5\) in our implementation. Let \(\alpha_{ij}\) denote the attention weight assigned from point \(p_i\) to token \(j\). We only retain token-point assignments for which token \(j\) appears in the top-\(K\) list of point \(p_i\).

\paragraph{\textbf{Active tokens.}}
A token \(j\) is considered \emph{active} for a given sample if it appears in the top-\(K\) attention list of at least one sampled surface point. Let
\begin{equation}
\mathcal{I}_j = \{ i \in \{1,\dots,M\} \mid j \in \mathrm{TopK}(p_i)\}
\end{equation}
be the set of sampled points associated with token \(j\). Tokens with \(|\mathcal{I}_j|=0\) are treated as inactive and excluded from the locality-score aggregation.

\paragraph{\textbf{Token centroid.}}
For each active token \(j\), we define its attention-weighted spatial centroid as
\begin{equation}
c_j
=
\frac{\sum_{i \in \mathcal{I}_j} \alpha_{ij} \, p_i}
{\sum_{i \in \mathcal{I}_j} \alpha_{ij}}.
\label{eq:token_centroid}
\end{equation}

\paragraph{Token spread.}
We then measure how spatially concentrated token \(j\)'s associated points are around the centroid using the attention-weighted root-mean-square distance
\begin{equation}
\sigma_j
=
\sqrt{
\frac{\sum_{i \in \mathcal{I}_j} \alpha_{ij}\,\|p_i-c_j\|^2}
{\sum_{i \in \mathcal{I}_j} \alpha_{ij}}
}.
\label{eq:token_spread}
\end{equation}
A smaller \(\sigma_j\) indicates that the token's attention is concentrated on a compact spatial region.

\paragraph{Locality score.}
To make the spread comparable across shapes of different scales, we normalize it by the diagonal length of the mesh bounding box:
\begin{equation}
d_{\mathrm{bbox}}
=
\left\|
\max_{i} p_i - \min_{i} p_i
\right\|_2.
\label{eq:bbox_diag}
\end{equation}
The locality score of token \(j\) is defined as
\begin{equation}
\mathrm{Locality}_j
=
1-\frac{\sigma_j}{d_{\mathrm{bbox}}}.
\label{eq:locality_score}
\end{equation}
Thus, a higher value indicates more spatially concentrated token behavior, while a lower value indicates more spatially diffuse attention.

\paragraph{\textbf{Aggregation.}}
For each sample, we compute the locality score for all active tokens and report the mean over active tokens. The final model-level statistic is reported as mean \(\pm\) standard deviation across all evaluated samples.

\paragraph{\textbf{Remark.}}
This metric does not directly measure geometric deformation under explicit token perturbation. Rather, it provides a consistent proxy for how spatially concentrated each token's decoder interaction is over the 3D surface. We use it here because it allows a unified comparison across different VecSet-based generators.

\subsection{Cross-model Comparison}
\label{app:locality_results}

Table~\ref{tab:locality_across_models} summarizes the comparison between Hunyuan3D and TripoSG.

\begin{table}[t]
\centering
\caption{Comparison of token activation and spatial locality across different VecSet-based 3D generative models.}
\label{tab:locality_across_models}
\begin{tabular}{lcc}
\toprule
 & Hunyuan3D & TripoSG \\
\midrule
Total tokens & 4096 & 2048 \\
Active tokens (avg) & \(283/4096\) (\(6.9\%\)) & \(2048/2048\) (\(100\%\)) \\
Locality score (avg) & \(0.9852 \pm 0.004\) & \(0.9747 \pm 0.017\) \\
\bottomrule
\end{tabular}
\end{table}

The two models exhibit markedly different activation regimes. Hunyuan3D activates only a small subset of tokens per sample, showing a sparse activation pattern in which each active token tends to cover a relatively broad region. In contrast, TripoSG activates nearly all tokens, corresponding to a much denser partition of the 3D surface. This indicates that the two models distribute representation across tokens in substantially different ways, despite sharing the same VecSet-based formulation.

Despite this difference, both models achieve similarly high locality scores. This means that token-point interactions remain strongly concentrated in local spatial regions under both sparse and dense activation regimes. Therefore, token-wise spatial locality is not merely an artifact of sparsity or a property specific to a single architecture, but appears to persist across different VecSet-based generators.

These findings strengthen the motivation of our method: if tokens correspond to spatially localized regions even across different VecSet-based generators, then explicit token-wise control is a principled way to improve localized editing. Rather than applying a uniform global steering signal, our results suggest that more precise editing can be achieved by selectively controlling tokens according to their localized geometric roles.

\begin{figure}[t]
\centering
\includegraphics[width=\linewidth]{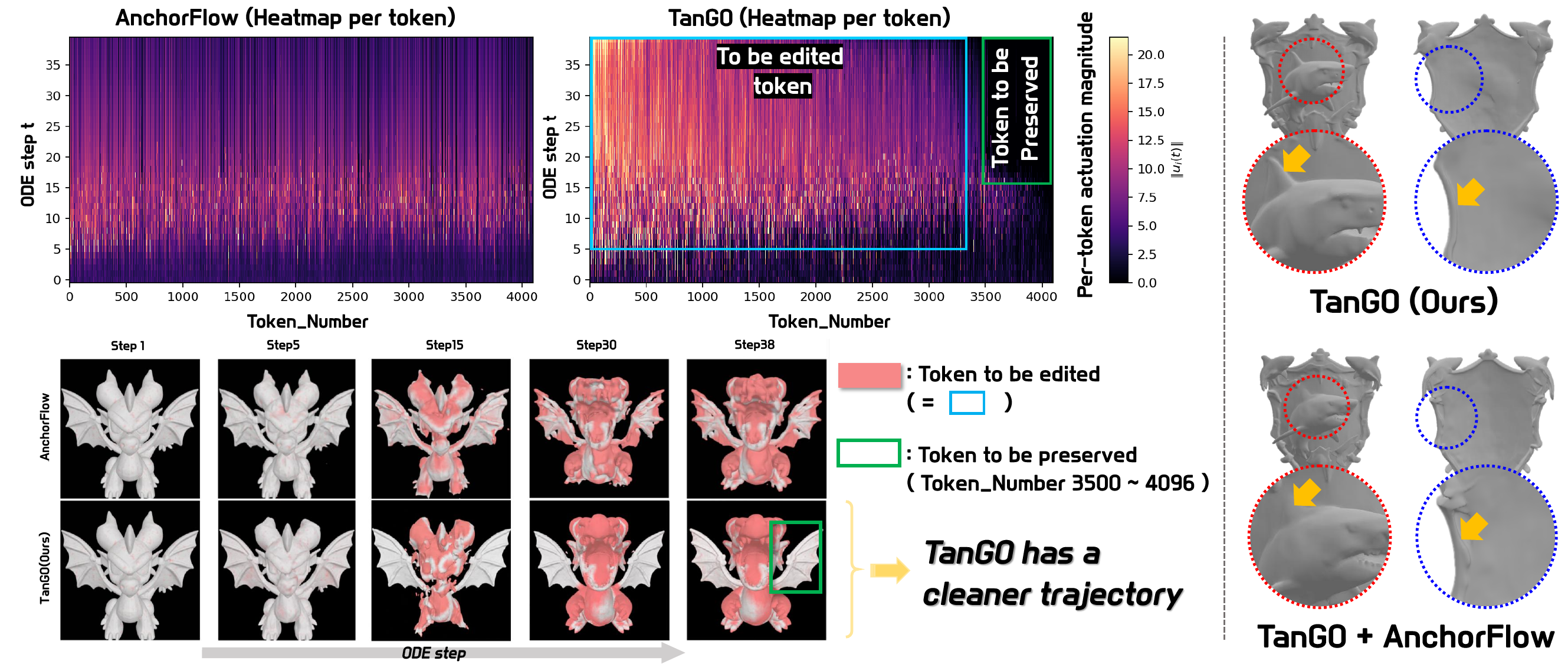}
\caption{\textbf{Per-token steering trajectory.} \textbf{(left)} Heatmaps of $\|u_i(t)\|$ and mesh decoding show TanGO has a cleaner trajectory. \textbf{(right)} Naive composition perturbs the preserved regions.}
\label{fig:fig10}
\end{figure}

\subsection{Token-wise Steering Trajectory}

We further visualize how TanGO modulates token-wise editing responses across ODE steps. Fig.~\ref{fig:fig10} shows the per-token actuation magnitude $|u_i(t)|$ together with intermediate mesh decoding results along the editing trajectory. Compared with AnchorFlow, TanGO produces a more selective actuation pattern: it concentrates stronger signals on tokens associated with regions to be edited, while suppressing responses on tokens corresponding to regions that should remain preserved.

This token-wise behavior leads to a cleaner editing trajectory. In AnchorFlow, the steering signal is distributed more uniformly across tokens, and broad regions of the mesh are already perturbed at early ODE steps. In contrast, TanGO progressively refines the intended target region while maintaining the surrounding geometry. These results support our claim that adaptive per-token steering provides more localized control than global trajectory-level guidance.

We also examine a naive combination of TanGO with AnchorFlow. As shown in Fig.~\ref{fig:fig10}, the direct composition does not further improve TanGO and can perturb preserved regions. This suggests that effective integration with anchor-based methods would require token-aware anchoring rather than a direct global-anchor composition.

\section{Detailed Proofs and Derivations}
\label{app:proofs}

In this appendix, we provide detailed derivations for the one-step optimal control
in Sec.~4.1, the vMF-based directional discrepancy in Sec.~4.2, and the mean-gain
normalization in Sec.~4.3.

\subsection{Closed-form Solution of the One-step Optimal Control}
\label{app:one_step_control}

Recall the instantaneous one-step control problem in Eq.~(4):
\begin{equation}
\max_{u_i} \;
d_i(t)\,\langle u_i, \Delta v_i(t) \rangle
-\frac{\rho}{2}\|u_i\|^2
\quad
\text{s.t. }
u_i \in \mathrm{span}\!\bigl(\Delta v_i(t)\bigr),
\label{eq:app_control_problem}
\end{equation}
where \(\rho > 0\) and \(d_i(t)\ge 0\).

\begin{proposition}
\label{prop:closed_form_control}
Assume \(\Delta v_i(t)\neq 0\). Then the optimization problem in
Eq.~\eqref{eq:app_control_problem} has the unique maximizer
\begin{equation}
u_i^*(t)=\frac{d_i(t)}{\rho}\,\Delta v_i(t).
\label{eq:app_closed_form_u}
\end{equation}
\end{proposition}

\begin{proof}
Since \(u_i \in \mathrm{span}(\Delta v_i(t))\), there exists a scalar \(g_i(t)\in\mathbb{R}\)
such that
\begin{equation}
u_i(t)=g_i(t)\,\Delta v_i(t).
\end{equation}
Substituting this into Eq.~\eqref{eq:app_control_problem}, we obtain
\begin{align}
J(g_i)
&=
d_i(t)\,\langle g_i(t)\Delta v_i(t), \Delta v_i(t)\rangle
-\frac{\rho}{2}\|g_i(t)\Delta v_i(t)\|^2 \\
&=
d_i(t)\,g_i(t)\|\Delta v_i(t)\|^2
-\frac{\rho}{2}g_i(t)^2\|\Delta v_i(t)\|^2 \\
&=
\|\Delta v_i(t)\|^2
\left(
d_i(t)\,g_i(t)-\frac{\rho}{2}g_i(t)^2
\right).
\end{align}
Since \(\|\Delta v_i(t)\|^2>0\), maximizing \(J(g_i)\) is equivalent to maximizing
\begin{equation}
\phi(g_i)=d_i(t)\,g_i-\frac{\rho}{2}g_i^2.
\end{equation}
This is a strictly concave quadratic function because
\begin{equation}
\phi''(g_i)=-\rho<0.
\end{equation}
Hence it has a unique global maximizer given by the stationary point:
\begin{equation}
\phi'(g_i)=d_i(t)-\rho g_i=0
\quad\Longrightarrow\quad
g_i^*(t)=\frac{d_i(t)}{\rho}.
\end{equation}
Therefore,
\begin{equation}
u_i^*(t)=g_i^*(t)\Delta v_i(t)=\frac{d_i(t)}{\rho}\Delta v_i(t).
\end{equation}
This proves Eq.~\eqref{eq:app_closed_form_u}.
\end{proof}

\begin{corollary}
\label{cor:degenerate_delta_v}
If \(\Delta v_i(t)=0\), then \(\mathrm{span}(\Delta v_i(t))=\{0\}\), and the unique feasible
point is \(u_i(t)=0\), which is therefore the unique optimizer.
\end{corollary}

\subsection{vMF-based Derivation of the Directional Discrepancy}
\label{app:vmf_derivation}

For the theoretical derivation, we first define the strictly normalized token-wise
directions
\begin{equation}
\tilde v_{\mathrm{src},i}(t)
=
\frac{v_{\mathrm{src},i}(t)}{\|v_{\mathrm{src},i}(t)\|},
\qquad
\tilde v_{\mathrm{tar},i}(t)
=
\frac{v_{\mathrm{tar},i}(t)}{\|v_{\mathrm{tar},i}(t)\|},
\label{eq:app_strict_normalization}
\end{equation}
whenever \(v_{\mathrm{src},i}(t)\neq 0\) and \(v_{\mathrm{tar},i}(t)\neq 0\).
Then
\(\tilde v_{\mathrm{src},i}(t), \tilde v_{\mathrm{tar},i}(t)\in\mathbb{S}^{d-1}\), and we define
\begin{equation}
\cos\theta_i(t)
=
\tilde v_{\mathrm{src},i}(t)^\top \tilde v_{\mathrm{tar},i}(t).
\label{eq:app_cos_theta}
\end{equation}

We model \(\tilde v_{\mathrm{tar},i}(t)\) as directional data concentrated around mean
direction \(\tilde v_{\mathrm{src},i}(t)\) using the von Mises--Fisher (vMF) density:
\begin{equation}
p(x\mid \mu,\kappa)
=
C_d(\kappa)\exp\!\bigl(\kappa\,\mu^\top x\bigr),
\qquad
x\in\mathbb{S}^{d-1},\;
\mu\in\mathbb{S}^{d-1},\;
\kappa\ge 0.
\label{eq:app_vmf}
\end{equation}

\begin{proposition}
\label{prop:vmf_affine_equivalence}
Let \(x=\tilde v_{\mathrm{tar},i}(t)\) and \(\mu=\tilde v_{\mathrm{src},i}(t)\). Then the
negative log-likelihood of the vMF density is affine-equivalent to
\(1-\cos\theta_i(t)\), i.e.,
\begin{equation}
-\log p\!\bigl(\tilde v_{\mathrm{tar},i}(t)\mid \tilde v_{\mathrm{src},i}(t),\kappa\bigr)
=
C'(\kappa)
+
\kappa\bigl(1-\cos\theta_i(t)\bigr),
\label{eq:app_vmf_affine}
\end{equation}
where \(C'(\kappa)=-\log C_d(\kappa)-\kappa\) is independent of
\(\tilde v_{\mathrm{tar},i}(t)\).
\end{proposition}

\begin{proof}
Starting from Eq.~\eqref{eq:app_vmf},
\begin{equation}
\log p(x\mid \mu,\kappa)
=
\log C_d(\kappa)+\kappa\,\mu^\top x.
\end{equation}
Hence
\begin{equation}
-\log p(x\mid \mu,\kappa)
=
-\log C_d(\kappa)-\kappa\,\mu^\top x.
\end{equation}
Substituting \(x=\tilde v_{\mathrm{tar},i}(t)\) and \(\mu=\tilde v_{\mathrm{src},i}(t)\), and using
Eq.~\eqref{eq:app_cos_theta}, we obtain
\begin{equation}
-\log p\!\bigl(\tilde v_{\mathrm{tar},i}(t)\mid \tilde v_{\mathrm{src},i}(t),\kappa\bigr)
=
-\log C_d(\kappa)-\kappa \cos\theta_i(t).
\end{equation}
Now rewrite the last term as
\begin{equation}
-\kappa \cos\theta_i(t)
=
\kappa\bigl(1-\cos\theta_i(t)\bigr)-\kappa.
\end{equation}
Therefore,
\begin{equation}
-\log p\!\bigl(\tilde v_{\mathrm{tar},i}(t)\mid \tilde v_{\mathrm{src},i}(t),\kappa\bigr)
=
\bigl[-\log C_d(\kappa)-\kappa\bigr]
+
\kappa\bigl(1-\cos\theta_i(t)\bigr),
\end{equation}
which proves Eq.~\eqref{eq:app_vmf_affine}.
\end{proof}

This motivates the directional discrepancy
\begin{equation}
d_i(t)=1-\cos\theta_i(t).
\label{eq:app_di}
\end{equation}

\subsection{Geometric Properties of the Discrepancy}
\label{app:geometric_properties}

\begin{proposition}
\label{prop:bounded_discrepancy}
For all \(i,t\), the discrepancy in Eq.~\eqref{eq:app_di} satisfies
\begin{equation}
d_i(t)\in[0,2].
\end{equation}
\end{proposition}

\begin{proof}
Since \(\tilde v_{\mathrm{src},i}(t)\) and \(\tilde v_{\mathrm{tar},i}(t)\) are unit vectors,
the Cauchy--Schwarz inequality gives
\begin{equation}
-1
\le
\tilde v_{\mathrm{src},i}(t)^\top \tilde v_{\mathrm{tar},i}(t)
=
\cos\theta_i(t)
\le 1.
\end{equation}
Subtracting from \(1\) yields
\begin{equation}
0 \le 1-\cos\theta_i(t)\le 2.
\end{equation}
By Eq.~\eqref{eq:app_di}, this is exactly
\begin{equation}
0\le d_i(t)\le 2.
\end{equation}
\end{proof}

\begin{proposition}
\label{prop:geometric_identity}
Let
\(a=\tilde v_{\mathrm{src},i}(t)\in\mathbb{S}^{d-1}\) and
\(b=\tilde v_{\mathrm{tar},i}(t)\in\mathbb{S}^{d-1}\).
Then
\begin{equation}
\|a-b\|^2=2\,d_i(t).
\label{eq:app_geometric_identity}
\end{equation}
\end{proposition}

\begin{proof}
Using the standard expansion of the squared Euclidean distance,
\begin{equation}
\|a-b\|^2
=
\|a\|^2+\|b\|^2-2a^\top b.
\end{equation}
Because \(a,b\in\mathbb{S}^{d-1}\), we have \(\|a\|=\|b\|=1\), and thus
\begin{equation}
\|a-b\|^2
=
1+1-2a^\top b
=
2(1-a^\top b).
\end{equation}
Since \(a^\top b=\cos\theta_i(t)\), we obtain
\begin{equation}
\|a-b\|^2
=
2\bigl(1-\cos\theta_i(t)\bigr)
=
2\,d_i(t).
\end{equation}
This proves Eq.~\eqref{eq:app_geometric_identity}.
\end{proof}

\begin{proposition}
\label{prop:monotonicity}
The discrepancy \(d_i(t)=1-\cos\theta_i(t)\) is monotonically increasing with respect
to \(\theta_i(t)\in[0,\pi]\).
\end{proposition}

\begin{proof}
Consider the function
\begin{equation}
f(\theta)=1-\cos\theta,
\qquad \theta\in[0,\pi].
\end{equation}
Its derivative is
\begin{equation}
f'(\theta)=\sin\theta.
\end{equation}
Since \(\sin\theta\ge 0\) for all \(\theta\in[0,\pi]\), the function \(f\) is monotonically
increasing on \([0,\pi]\). Therefore, \(d_i(t)\) increases as the angular mismatch
\(\theta_i(t)\) increases.
\end{proof}

\subsection{Mean-gain Normalization}
\label{app:mean_gain}

Recall that the raw token-wise gain is defined as
\begin{equation}
g_i(t)=\frac{d_i(t)}{\rho},
\qquad
\bar g(t)=\frac{1}{N}\sum_{i=1}^N g_i(t),
\end{equation}
and the effective global scale is
\begin{equation}
\lambda_{\mathrm{eff}}(t)
=
\lambda\cdot \frac{\eta}{\bar g(t)+\varepsilon}.
\label{eq:app_lambda_eff}
\end{equation}
The final control is
\begin{equation}
u_i(t)=\lambda_{\mathrm{eff}}(t)\,g_i(t)\,\Delta v_i(t).
\label{eq:app_final_control}
\end{equation}

\begin{proposition}
\label{prop:mean_effective_gain}
The mean effective gain satisfies
\begin{equation}
\frac{1}{N}\sum_{i=1}^N \lambda_{\mathrm{eff}}(t)\,g_i(t)
=
\lambda\cdot \frac{\eta\,\bar g(t)}{\bar g(t)+\varepsilon}.
\label{eq:app_mean_effective_gain}
\end{equation}
\end{proposition}

\begin{proof}
Because \(\lambda_{\mathrm{eff}}(t)\) does not depend on the token index \(i\),
\begin{equation}
\frac{1}{N}\sum_{i=1}^N \lambda_{\mathrm{eff}}(t)\,g_i(t)
=
\lambda_{\mathrm{eff}}(t)\cdot \frac{1}{N}\sum_{i=1}^N g_i(t)
=
\lambda_{\mathrm{eff}}(t)\,\bar g(t).
\end{equation}
Substituting Eq.~\eqref{eq:app_lambda_eff} gives
\begin{equation}
\lambda_{\mathrm{eff}}(t)\,\bar g(t)
=
\lambda\cdot \frac{\eta}{\bar g(t)+\varepsilon}\,\bar g(t)
=
\lambda\cdot \frac{\eta\,\bar g(t)}{\bar g(t)+\varepsilon},
\end{equation}
which proves Eq.~\eqref{eq:app_mean_effective_gain}.
\end{proof}

\begin{remark}
When \(\varepsilon=0\), Eq.~\eqref{eq:app_mean_effective_gain} reduces exactly to
\begin{equation}
\frac{1}{N}\sum_{i=1}^N \lambda_{\mathrm{eff}}(t)\,g_i(t)=\lambda \eta.
\end{equation}
For \(\varepsilon>0\), the mean effective gain becomes a regularized version of the
target level \(\lambda \eta\), which stabilizes the timestep-wise fluctuation induced by
variation in \(\bar g(t)\).
\end{remark}

\subsection{Implementation Note on Regularized Normalization}
\label{app:regularized_normalization}

In practice, for numerical stability, we use the regularized normalization
\begin{equation}
\hat v_{\mathrm{src},i}(t)
=
\frac{v_{\mathrm{src},i}(t)}{\|v_{\mathrm{src},i}(t)\|+\varepsilon},
\qquad
\hat v_{\mathrm{tar},i}(t)
=
\frac{v_{\mathrm{tar},i}(t)}{\|v_{\mathrm{tar},i}(t)\|+\varepsilon}.
\end{equation}
These vectors are not exactly unit-normalized in the strict mathematical sense.
Therefore, the vMF derivation and geometric identities above are stated for the ideal
unit-normalized directions \(\tilde v\), while the \(\hat v\)-based form should be viewed
as a numerically stable approximation used in implementation.

\section{Construction of the TanGOEdit Dataset}
\label{app:tangoedit_dataset}

To enable a controlled and fine-grained evaluation of localized text-driven 3D editing, we construct \textbf{TanGOEdit}, a curated benchmark consisting of 100 editing instances spanning five edit categories: \textbf{Add}, \textbf{Replace}, \textbf{Action Change}, \textbf{Style Change}, and \textbf{Remove}. Our dataset design is inspired by recent multi-stage 3D editing benchmark construction pipelines. In particular, we follow the overall philosophy of building a compact but diverse evaluation set with a roughly balanced category distribution, while adopting a VLM-assisted instruction construction strategy inspired by Nano3D. Rather than focusing on scale, our goal is to build a reliable benchmark that covers both rigid local edits and broader identity-preserving transformations.

Each sample in TanGOEdit consists of a \textit{source 3D asset}, a \textit{source rendering}, and a corresponding \textit{editing instruction}. The five categories are chosen to jointly cover representative localized structural edits, appearance-level modifications, and non-rigid semantic transformations. The final benchmark contains 100 curated examples, distributed across the five categories in a roughly balanced manner.
\paragraph{\textbf{Step 1: Source Asset Collection.}}
We first collect candidate 3D assets with clear object identity and visually interpretable geometry. For each asset, we inspect rendered views and retain examples whose main structure and editable regions are sufficiently visible for text-guided editing. We intentionally include diverse object types so that the benchmark does not overfit to a narrow semantic domain.
\paragraph{\textbf{Step 2: Editing Instruction Construction.}}
For each selected asset, we construct an editing instruction with the assistance of a vision-language model. Inspired by the instruction generation strategy used in Nano3D, we use Qwen-VL-2.5 to generate candidate instructions from rendered views and their visual context. While prior template-based pipelines mainly focus on \textit{add}, \textit{remove}, and \textit{replace} operations, we extend the instruction space to five categories by additionally introducing \textit{Action Change} and \textit{Style Change}. The generated instruction is required to be specific, semantically grounded, and identity-preserving, so that the requested edit modifies only the intended part, attribute, pose, or appearance without replacing the entire object.

For rigid edit types such as \textbf{Add}, \textbf{Replace}, and \textbf{Remove}, the instruction explicitly specifies the target entity and, whenever possible, the target region. For non-rigid edit types such as \textbf{Action Change} and \textbf{Style Change}, the instruction describes a plausible transformation that preserves the core identity of the source object. This design allows TanGOEdit to evaluate both precise local control and broader semantic transformation within a unified benchmark. Below, we provide the exact prompt template used for instruction construction.

\bluebox{Step 2: Editing Instruction Construction}{

\textbf{Task Definition.} Construction of a balanced 3D editing benchmark.

\textbf{Objective.} Given rendered views of a source 3D asset, assign the asset to one of five predefined editing categories and generate a single editing instruction that specifies a realistic and visually grounded modification.

\textbf{Editing Categories.} Add, Replace, Action Change, Style Change, Remove.

\textbf{Procedure and Requirements.}
\begin{itemize}
    \item  {Visual Analysis} : Inspect the rendered views and identify the main object identity, editable attributes, and visible regions.
    \item  {Category Assignment} : Assign the most semantically appropriate editing category among the five predefined categories.  
    \item  {Instruction Generation} : Generate one concise editing instruction that clearly describes the desired modification.
    \item  {Specificity} : The instruction must be explicit and unambiguous, and should clearly identify the target object, part, attribute, or transformation.
    \item {Identity Preservation} : The edit must preserve the core identity of the source object and should not amount to a complete replacement of the entire asset.
    \item {Localizability} : The requested change should be visually grounded in a region or attribute that can be reasonably evaluated from rendered views.
    \item {Diversity} : The final dataset should maintain diversity in semantics and edit patterns while preserving an approximately balanced distribution across the five categories.
    \item {Category-specific Constraints} : For Action Change, the edit should describe a plausible pose or action change while preserving object identity. For Style Change, the instruction should modify appearance, color, material, or visual style without changing the main object semantics.
\end{itemize}

\textbf{Output Format} : For each source asset, return (1) the assigned editing category and (2) the final editing instruction.

}

\paragraph{\textbf{Step 3: Quality Control and Final Sample Selection.}}
After candidate instructions are generated, we perform a verification and refinement stage to improve dataset quality. Ambiguous, visually unverifiable, semantically inconsistent, or overly destructive instructions are removed or rewritten. We further discard cases in which the edited content is not clearly supported by the visible source rendering or where the intended modification is difficult to evaluate reliably. Finally, we select 100 high-quality samples with a roughly balanced distribution across the five edit categories, resulting in a compact but challenging benchmark for localized 3D editing.

As a result, TanGOEdit contains a compact yet diverse set of editing scenarios that jointly test multiple aspects of 3D editing performance. Local structural edits such as Add, Replace, and Remove primarily evaluate whether a method can confine changes to the intended region without introducing collateral deformation, while Action Change and Style Change further test whether the method can handle broader semantic transformations without losing global coherence. This combination makes TanGOEdit suitable for evaluating both localized controllability and robustness under more challenging edits.

Overall, TanGOEdit is designed to evaluate whether a method can simultaneously achieve \textbf{(1) precise local modification}, \textbf{(2) faithful preservation of unedited regions}, and \textbf{(3) global structural coherence} under both rigid and non-rigid editing scenarios.

\section{More Qualitative Results}
We provide additional qualitative comparisons in Figs.~\ref{fig:fig11}, \ref{fig:fig12} and \ref{fig:fig13} to further examine the behavior of TanGO under diverse text-driven 3D editing scenarios. The examples cover both rigid local edits, such as object addition, replacement, and removal, and broader non-rigid transformations, including pose and style changes. Across these cases, TanGO consistently produces edits that are well localized to the intended region while preserving the geometry and appearance of unedited parts.

Compared with prior methods, our approach more reliably avoids common failure modes such as under-editing, unintended deformation of neighboring parts, and leakage of the editing effect into regions that should remain unchanged. This behavior is particularly evident in the highlighted regions, where baseline methods often either fail to modify the target part sufficiently or introduce collateral distortions around the edited area. In contrast, TanGO better maintains part boundaries and global structural coherence, even when the required transformation is relatively large.

These additional results support our main claim that token-wise tangent-space control provides a more precise mechanism for localized 3D editing. By selectively strengthening edits in target regions while suppressing unnecessary changes elsewhere, TanGO achieves a favorable balance between editability and preservation across a broad range of editing types.

\begin{figure*}[p]
  \begin{center}
    \centerline{\includegraphics[width=\textwidth]{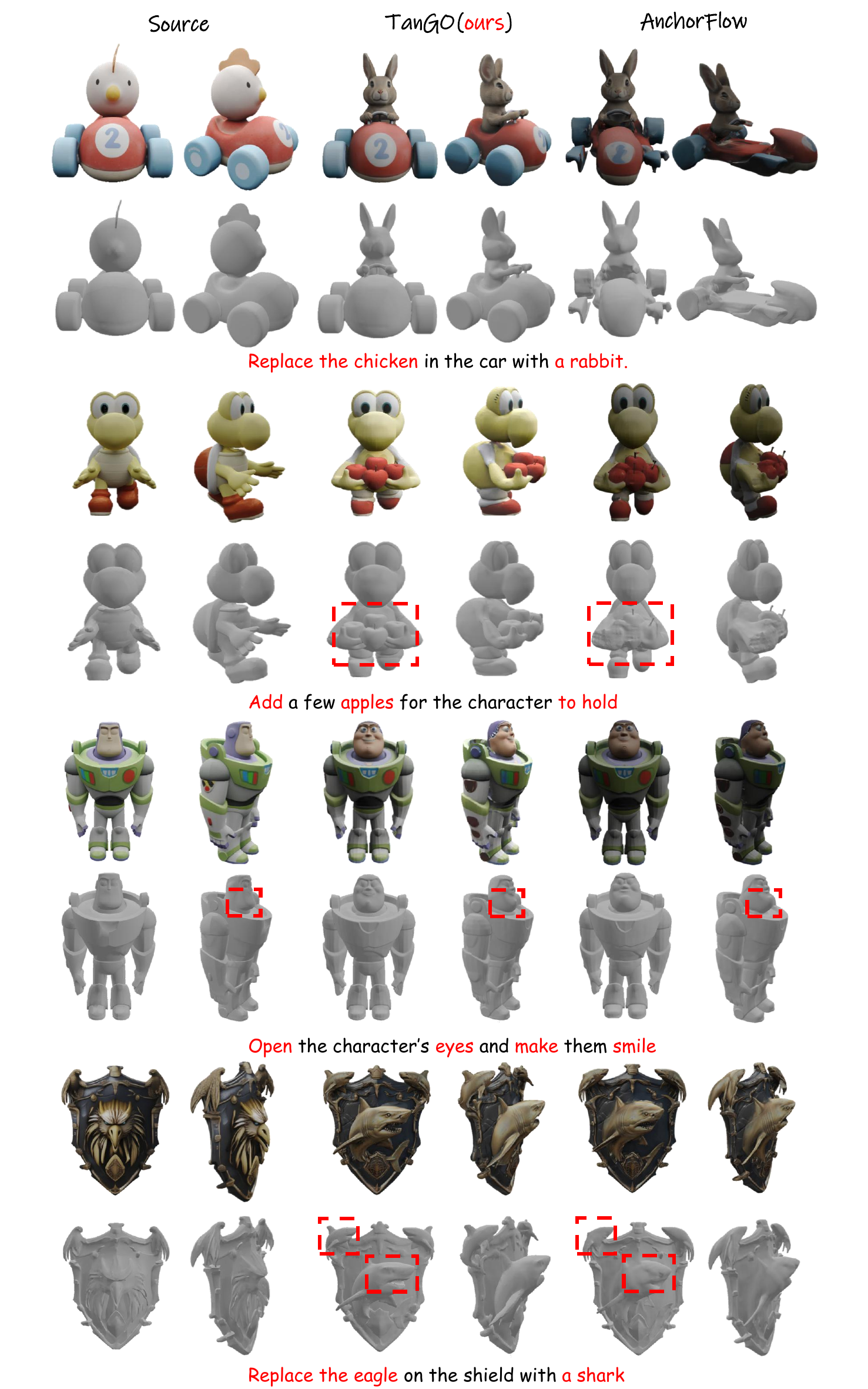}}
    \caption{Additional qualitative results. TanGO achieves more localized edits.}
    \label{fig:fig11}
  \end{center}
\end{figure*}

\begin{figure*}[p]
  \begin{center}
    \centerline{\includegraphics[width=\textwidth]{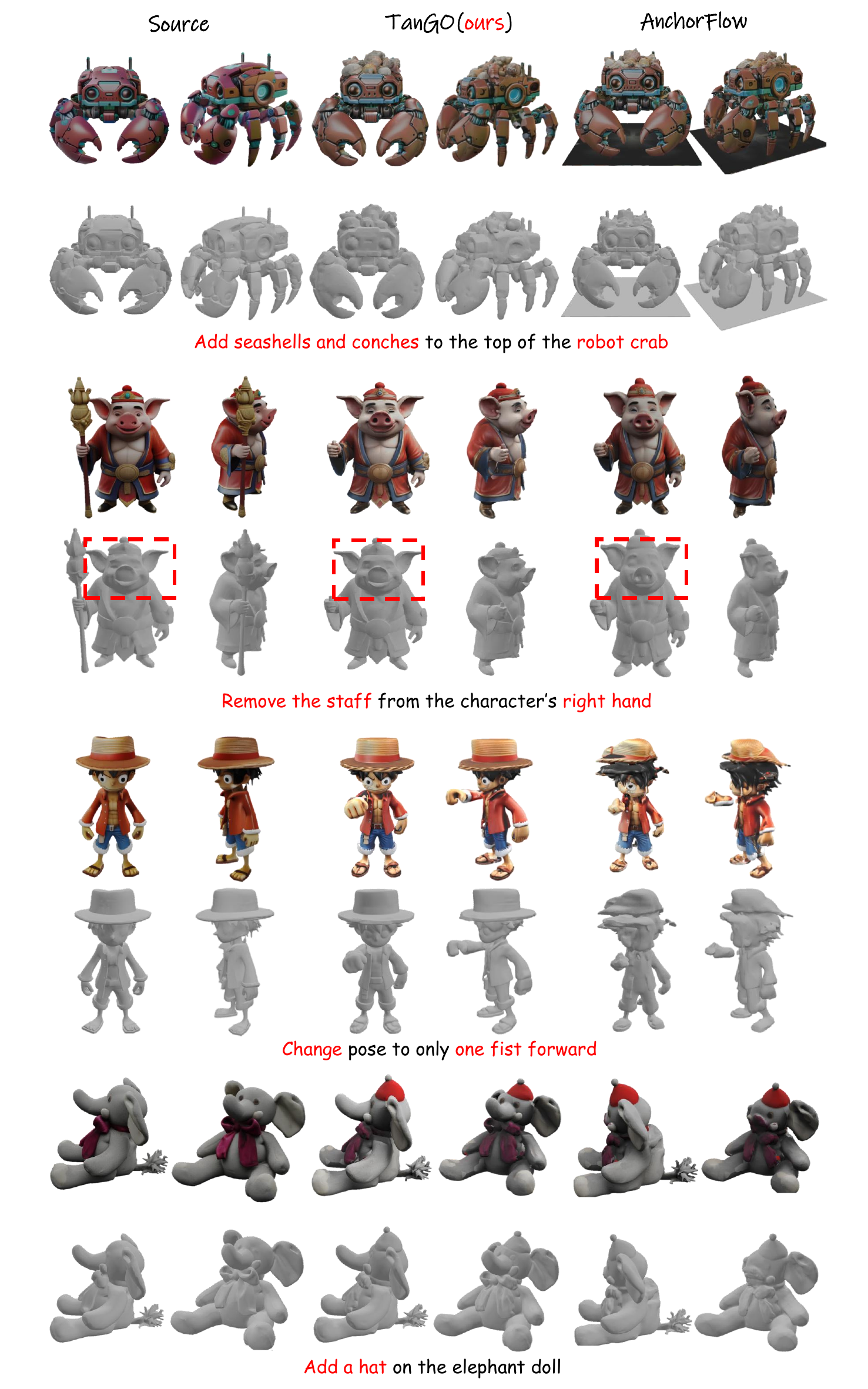}}
    \caption{Additional qualitative results (continued)}
    \label{fig:fig12}
  \end{center}
\end{figure*}

\begin{figure*}[p]
  \begin{center}
    \centerline{\includegraphics[width=\textwidth, height=20cm]{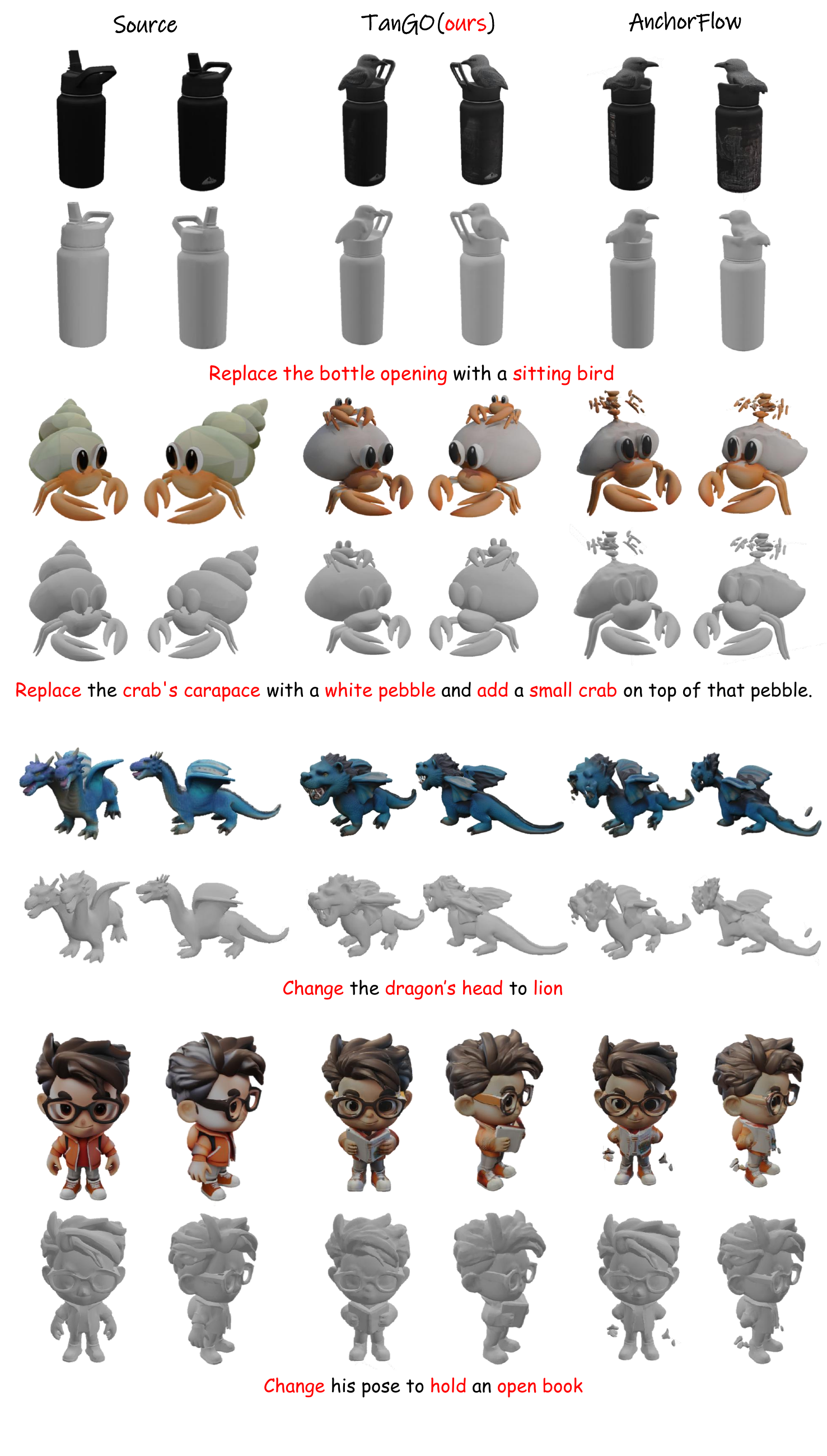}}
    \caption{Additional qualitative results (continued)}
    \label{fig:fig13}
  \end{center}
\end{figure*}

\clearpage

\section{More Quantitative Results}
\label{app:more_qualitative}

\subsection{Mask-aware 3D Evaluation}
To complement the main TanGOEdit evaluation, we conduct an additional mask-aware 3D evaluation on a separate 150-sample benchmark. This benchmark consists of 100 samples from Edit3DBench and 50 TanGOEdit samples annotated in Blender. We use this separate protocol because mask-aware geometric preservation metrics require ground-truth 3D edit masks, which are not available for the full TanGOEdit benchmark. In addition to the rendered-view metrics DINO-I, CLIP-I, and CLIP-T, we report Uni3D$_{\mathrm{txt}}$ for 3D-native text-shape alignment, and CD$_{\mathrm{preserve}}$ and NC$_{\mathrm{preserve}}$ for mask-aware geometric preservation.

\begin{table}[htb!]
\centering
\scriptsize
\caption{Additional baseline comparisons with 3D metrics.}
\label{tab:geometry_metrics}
\resizebox{\linewidth}{!}{
\begin{tabular}{l c c c c c c}
\toprule
Method & DINO-I $\uparrow$ & CLIP-I $\uparrow$ & CLIP-T $\uparrow$ & Uni3D$_{\mathrm{txt}}$ $\uparrow$ & CD$_{\mathrm{preserve}}$ $\downarrow$ & NC$_{\mathrm{preserve}}$ $\uparrow$ \\
\midrule
MVEdit       & 0.5931 & 0.4986 & 0.1311 & 0.0858 & 0.0381 & 0.5449 \\
EditP23      & 0.5948 & 0.5002 & 0.1309 & 0.0871 & 0.0311 & 0.5804 \\
FlowEdit     & 0.6118 & 0.7213 & 0.2085 & 0.2239 & 0.0398 & 0.5775 \\
AnchorFlow   & 0.6459 & 0.7478 & 0.2241 & 0.2432 & 0.0382 & 0.5669 \\
VoxHammer    & 0.6482 & 0.7641 & 0.2273 & 0.2459 & 0.0355 & 0.5685 \\
Nano3D       & 0.6172 & 0.7583 & 0.2194 & 0.2375 & 0.0381 & 0.5638 \\
\midrule
\rowcolor{green!15}
\textbf{TanGO (Ours)} & \textbf{0.6623} & \textbf{0.7718} & \textbf{0.2388} & \textbf{0.2593} & \textbf{0.0295} & \textbf{0.5835} \\
\bottomrule
\end{tabular}
}
\end{table}

As shown in Table~\ref{tab:geometry_metrics}, TanGO outperforms all baselines, including VoxHammer and Nano3D, across both rendered-view metrics and 3D-aware metrics. The improvement in Uni3D$_{\mathrm{txt}}$ indicates better 3D-native text-shape alignment, while the lower CD${\mathrm{preserve}}$ and higher NC$_{\mathrm{preserve}}$ show stronger preservation of unedited geometry. These results provide additional evidence that TanGO improves localized editing quality not only in rendered-view evaluation but also under mask-aware geometric assessment.

\subsection{Ablation on Per-token Gain Design}

To verify that the proposed directional demand is not merely an arbitrary token-wise rescaling, we compare TanGO with simpler gain designs. All variants use the same update form $u_i(t)=\lambda_{\mathrm{eff}}(t)g_i(t)\Delta v_i(t)$, but differ in how the per-token gain $g_i(t)$ is defined. We compare against global scaling, L2-magnitude-based gain, and inner-product-based gain. Global scaling ignores token-wise editing demand, while L2 magnitude and inner-product gains can be affected by raw velocity scale or unnormalized similarity. In contrast, TanGO uses a bounded and scale-invariant directional discrepancy.

\begin{table}[t]
\centering
\scriptsize
\caption{Ablation against simple per-token gain alternatives.}

\label{tab:per_token_ablation}
\resizebox{\linewidth}{!}{
\begin{tabular}{l c c c c c c}
\toprule
Method ($g_i$) 
& DINO-I $\uparrow$ 
& CLIP-I $\uparrow$ 
& CLIP-T $\uparrow$
& Uni3D$_{\mathrm{txt}}$ $\uparrow$ 
& CD$_{\mathrm{preserve}}$ $\downarrow$ 
& NC$_{\mathrm{preserve}}$ $\uparrow$ \\
\midrule
Global Scaling \quad ($g_i = \mathrm{const.}$)                                         & 0.6145 & 0.7230 & 0.2215 & 0.2304 & 0.0412 & 0.5120 \\
L2 Magnitude \quad ($g_i = \|v_{\mathrm{tar},i} - v_{\mathrm{src},i}\|$)  & 0.5821 & 0.6894 & 0.2044 & 0.2118 & 0.0567 & 0.4855 \\
Inner-Product \quad ($g_i = \langle v_{\mathrm{src},i}, v_{\mathrm{tar},i} \rangle$)  & 0.5903 & 0.7012 & 0.2109 & 0.2187 & 0.0531 & 0.4932 \\
\midrule
\rowcolor{green!15}
\textbf{TanGO (Ours)} & \textbf{0.6623} & \textbf{0.7718} & \textbf{0.2388} & \textbf{0.2593} & \textbf{0.0295} & \textbf{0.5835} \\
\bottomrule
\end{tabular}
}
\end{table}

As shown in Table~\ref{tab:per_token_ablation}, TanGO consistently achieves the best performance across all metrics. The global scaling baseline performs worse because it applies the same gain to all tokens regardless of local editing demand. The L2-magnitude and inner-product variants also underperform, suggesting that raw velocity magnitude or unnormalized similarity is less reliable for identifying tokens that require stronger editing. These results support the use of TanGO's direction-only demand for both semantic alignment and geometric preservation.

\subsection{Cross-model Quantitative Generalization}
To further examine whether TanGO generalizes beyond Hunyuan3D 2.1, we apply the same token-wise steering mechanism to TripoSG, another VecSet-based 3D foundation model. We compare TanGO with AnchorFlow under the same rendered-view and mask-aware 3D evaluation metrics.

\begin{table}[h] 
\centering
\scriptsize
\caption{Cross-model generalization results.}
\label{tab:cross_model_generalization}
\resizebox{\linewidth}{!}{
\begin{tabular}{l c c c c c c}
\toprule
& DINO-I $\uparrow$ 
& CLIP-I $\uparrow$ 
& CLIP-T $\uparrow$
& Uni3D$_{\mathrm{txt}}$ $\uparrow$ 
& CD$_{\mathrm{preserve}}$ $\downarrow$ 
& NC$_{\mathrm{preserve}}$ $\uparrow$ \\
\midrule
AnchorFlow + Hunyuan3D 2.1 & 0.6459 & 0.7478 & 0.2241 & 0.2432 & 0.0382 & 0.5669 \\
\rowcolor{green!15}
\textbf{TanGO + Hunyuan3D 2.1}      & \textbf{0.6623} & \textbf{0.7718} & \textbf{0.2388} & \textbf{0.2593} & \textbf{0.0295} & \textbf{0.5835} \\
AnchorFlow + TripoSG       &0.6402 &0.7374 &0.2219 &0.2397 &0.0399 &0.5542 \\
TanGO + TripoSG            &0.6523 &0.7511 &0.2315 &0.2453 &0.0343 &0.5699 \\
\bottomrule
\end{tabular}
}
\end{table}

As shown in Table~\ref{tab:cross_model_generalization}, TanGO consistently improves over AnchorFlow on both Hunyuan3D 2.1 and TripoSG. These results indicate that the proposed token-wise tangent-space steering is not restricted to a single VecSet-based backbone, and can generalize across different 3D foundation models.

\begin{figure*}[htb!]
  \begin{center}
    \centerline{\includegraphics[width=\textwidth]{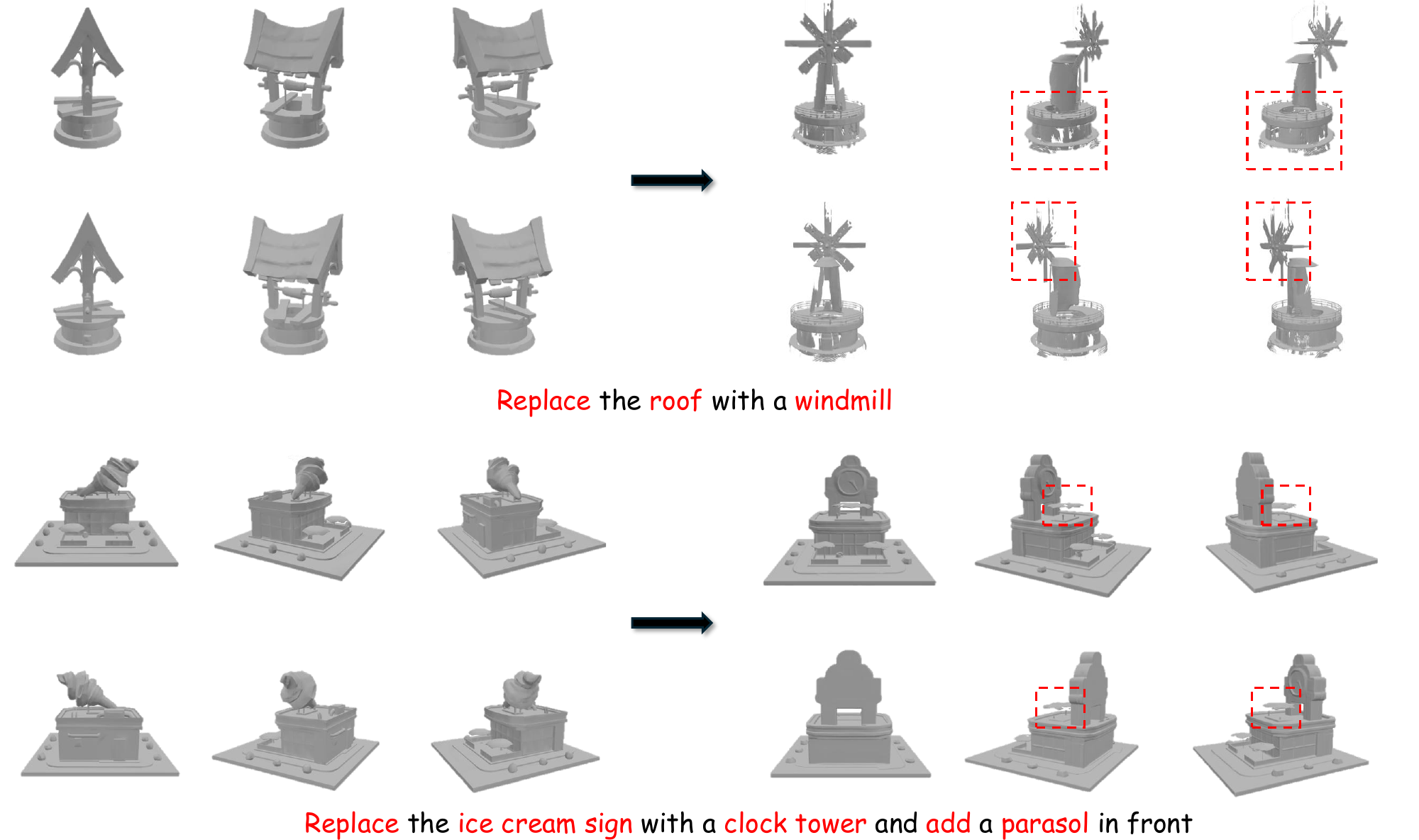}}
    \caption{\textbf{Failure cases:} The mesh breaks or distorts on thin structures like windmill blades or parasols.}
    \label{fig:fig14}
  \end{center}
\end{figure*}

\section{Limitations and Future Work}
While TanGO demonstrates robust capabilities in localized, training-free 3D editing, it inherits certain limitations from its underlying 3D generative framework. First, because our method operates entirely within the latent space of a pre-trained 3D foundation model, the overall quality and the preservation of original details in unedited regions are fundamentally bottlenecked by the reconstruction performance of the model's 3D Variational Autoencoder (VAE). Even when our tangent-space guidance successfully isolates and protects unedited tokens from modification, any high-frequency geometric or texture details that the VAE fails to accurately reconstruct during the initial encoding and final decoding stages will inevitably be lost. Consequently, the upper bound of unedited detail preservation is strictly dictated by the fidelity of the VAE, rather than the precision of our editing algorithm. Second, TanGO exhibits limited editing proficiency when dealing with extremely thin or highly intricate geometric structures as shown in Fig.~\ref{fig:fig14}. In current VecSet-based token representations, highly delicate features such as thin wires, fine meshes, or sharply protruding parts are often compressed into a remarkably small number of tokens or smoothed out by surrounding spatial contexts. As a result, applying precise tangent-space guidance to these structurally sparse regions becomes challenging. This limitation can sometimes lead to topological artifacts, structural collapse, or instances where the model ignores the text instruction for those specific delicate areas.

Future work could address these issues by integrating higher-resolution 3D tokenizers or adopting hierarchical latent representations that better capture and manipulate fine-grained topology without losing the advantages of training-free guidance.

\section{User Study Details}
To further validate the perceptual quality and editing precision of our method, we conducted a user study comparing TanGO with the strongest baseline, AnchorFlow. This study aims to provide a qualitative assessment of the generated 3D assets from a human perspective, complementing the quantitative metrics reported in Table 2 of the main paper.

\subsection{Participants and Setup}
We recruited 22 participants to evaluate the editing results. To ensure a comprehensive evaluation, we randomly sampled 20 text-driven 3D editing pairs from our TanGOEdit benchmark. These samples were carefully selected to encompass all five editing categories: Add, Replace, Action Change, Style Change, and Remove. 

\subsection{Evaluation Criteria}
Participants were asked to evaluate the generated 3D models based on three distinct criteria, which directly correspond to the dimensions reported in Table 2:
\begin{itemize}
    \item \textbf{Prompt Alignment:} Which model better reflects the target editing instruction while accurately applying the intended semantic changes?
    \item \textbf{Visual Quality:} Which model exhibits fewer visual artifacts (e.g., surface collapse, mesh tearing) and higher overall rendering aesthetics?
    \item \textbf{Shape Preservation:} Which model better preserves the unedited regions and the core identity of the original source 3D asset without introducing collateral deformations?
\end{itemize}

\subsection{Procedure}
The study was conducted using a Two-Alternative Forced Choice (2AFC) format\cite{macmillan2004detection}. For each test case, participants were presented with the source 3D rendering, the target text prompt, and the edited 3D models from both TanGO and AnchorFlow. To ensure a fair comparison, all 3D models were rendered from identical, representative camera viewpoints. The display order of the methods (Left/Right) was strictly randomized for each question to prevent any spatial selection bias. Participants were instructed to select the model that best satisfied each of the three criteria independently. The final preference percentages reported in the main paper were calculated by aggregating the votes across all 22 users and test cases.

\end{document}